\documentclass[conference]{IEEEtran}
\IEEEoverridecommandlockouts

\usepackage[utf8]{inputenc}

\usepackage{amsmath, amsthm, amssymb, amsfonts} 
\usepackage{nicefrac}
\usepackage{mathtools}
\usepackage{bm, bbm}
\usepackage{thm-restate}
\usepackage[scr=boondoxo,scrscaled=1.05]{mathalfa}

\usepackage{cite}

\usepackage{subfig}
\usepackage{booktabs}                                  

\usepackage{enumitem}

\usepackage{xcolor}
\usepackage{comment}

\usepackage{hyperref}
\usepackage{url}

\usepackage{blindtext}

\definecolor{orange}{rgb}{1,0.4,0.0}

\DeclarePairedDelimiterXPP{\KL}[2]{D_\textnormal{KL}}{(}{)}{}{%
#1\:\delimsize\|\:#2%
}

\DeclarePairedDelimiterXPP\Prob[1]{\mathbb{P}}{\lbrace}{\rbrace}{}{

#1}

\DeclarePairedDelimiterXPP{\lnorm}[2]{}{\lVert}{\rVert}{_{#2}}{#1}

\newcommand{\cS}{\ensuremath{\mathcal{S}}}
\newcommand{\cT}{\ensuremath{\mathcal{T}}}

\newcommand{\cX}{\ensuremath{\mathcal{X}}}
\newcommand{\cY}{\ensuremath{\mathcal{Y}}}

\newtheoremstyle{mytheoremstyle} 
    {\topsep}                    
    {\topsep}                    
    {\itshape}                   
    {}                           
    {\bf}                        
    {.}                          
    {.5em}                       
    {}  

\theoremstyle{mytheoremstyle}
\newtheorem{lemma}{Lemma}

\newtheorem{proposition}{Proposition}
\newtheorem{definition}{Definition}

\newtheorem{remark}{Remark}

\begin{document}
\title{A Variational Approach to Privacy and Fairness} 

 \author{%
   \IEEEauthorblockN{Borja Rodríguez-Gálvez, Ragnar Thobaben, and Mikael Skoglund}
   \IEEEauthorblockA{Division of Information Science and Engineering (ISE)\\
        KTH Royal Institute of Technology\\
        Email: \{borjarg,ragnart,skoglund\}@kth.se}
    \thanks{This paper was previously presented at the non-archival PPAI workshop from the AAAI 2021 conference. This work was funded in part by the Swedish research council under contract 2019-03606. A full version of this paper is accessible at \url{https://arxiv.org/abs/2006.06332} and the code is available at \url{https://github.com/burklight/VariationalPrivacyFairness}.}
 }

\maketitle

\begin{abstract}
   In this article, we propose a new variational approach to learn private and/or fair representations. This approach is based on the Lagrangians of a new formulation of the privacy and fairness optimization problems that we propose. In this formulation, we aim to generate representations of the data that keep a prescribed level of the relevant information that is not shared by the private or sensitive data, while minimizing the remaining information they keep. The proposed approach (i) exhibits the similarities of the privacy and fairness problems, (ii) allows us to control the trade-off between utility and privacy or fairness through the Lagrange multiplier parameter, and (iii) can be comfortably incorporated to common representation learning algorithms such as the VAE, the $\beta$-VAE, the VIB, or the nonlinear IB.
\end{abstract}

\section{Introduction}
\label{sec:introduction}

Currently, many systems rely on machine learning algorithms to make decisions and draw inferences. That is, they use previously existing data in order to shape some stage of their decision or inference mechanism. Usually, this data contains private or sensitive information, e.g., the identity of the person from which a datum was collected or their membership to a minority group. Thus, an important problem occurs when the data used to train such algorithms leaks this information to the system, contributing to unfair decisions or to a privacy breach.

When the content of the private information is arbitrary and the task of the system is not defined, the problem is reduced to learning \emph{private representations} of the data; i.e., representations that are informative of the data (utility), but are not informative of the private information. Then, these representations can be employed by any system with a controlled private information leakage. If the informativeness is measured by the 
mutual information, the problem of generating private representations is known as the \emph{privacy funnel} (PF) \cite{du2012privacy, makhdoumi2014information}. 

When the task of the system is known, then the aim is to design strategies so that the system performs such a task efficiently while employing or leaking little sensitive information. The field of \emph{algorithmic fairness} has extensively studied this problem, especially for classification tasks and categorical sensitive information, c.f., \cite{dwork2012fairness, zafar2015fairness, chouldechova2018frontiers, corbett2018measure}. An interesting approach is that of learning \emph{fair representations} \cite{zemel2013learning, edwards2015censoring, zhao2019conditional}, where similarly to their private counterparts, the representations are informative of the task, but contain little sensitive information.

There is a compromise between information leakage and utility when designing private representations \cite{makhdoumi2014information}. Similarly
, it has been shown empirically \cite{calders2009building} and theoretically \cite{zhao2019inherent} that there is a trade-off between fairness and utility.

In this work, we investigate the trade-off between utility and privacy and between utility and fairness in terms of mutual information (details about the choice of mutual information in Appendix~\ref{app:why-mi}). More specifically, we aim at maintaining a certain level of the information about the data (for privacy) or the task (for fairness) that is not shared by the senstive attributes, while minimizing all the other information. We name these two optimization problems the \emph{conditional privacy funnel} (CPF) and the \emph{conditional fairness bottleneck} (CFB) due to their similarities with the PF \cite{makhdoumi2014information}, the information bottleneck (IB) \cite{tishby2000information}, and the recent conditional entropy bottleneck (CEB) \cite{fischer2020conditional}. 

We tackle both optimization problems with a variational approach based on their Lagrangian. For the privacy problem, we show that the minimization of the Lagrangians of the CPF and the PF is equivalent (see Appendix \ref{app:lagrangians_pf_cpf_equivalence}), meaning that the proposed approach also attempts at solving the PF. Moreover, this approach improves over current variational approaches to the PF by respecting the problem's Markov chain in the encoder distribution.

Finally, the resulting approaches for privacy and fairness can be implemented with little modification to common algorithms for representation learning like the variational autoencoder (VAE) \cite{kingma2013auto}, the $\beta$-VAE \cite{higgins2017beta}, the variational information bottleneck (VIB) \cite{nan2020variational}, or the nonlinear information bottleneck \cite{kolchinsky2019nonlinear}. Therefore, it facilitates the incorporation of private and fair representations in current applications (see Appendices~\ref{app:related_work} and \ref{app:modification_common_algorithms} for an extended related section and a guide on how to modify these algorithms).

We demonstrate our results both in the Adult dataset~\cite{Dua2019} and a toy dataset based on MNIST \cite{lecun1998gradient}. Further experiments on the COMPAS dataset \cite{dieterich2016compas} can be found in Appendix \ref{app:experiments}.

\section{Methods}
\label{sec:methods}

In this section we present our approach. First, we introduce the proposed models for the privacy and fairness problems. Then, we show a suitable Lagrangian formulation. Finally, we describe a variational approach to solve both problems.
\subsection{Problem formulations}
\label{subsec:problem_formulation}
\subsubsection{Privacy: the conditional privacy funnel (CPF)}
\label{subsubsec:privacy_formulation}
Suppose we want to share some data $X \in \cX$ so that third-parties can make statistical analyses and draw inferences from them. However, these data $X$ contains information about some private attributes $S \in \cX$ that we want to protect. For this reason, we encode the data $X$ into the representation $Y \in \cY$, forming the Markov chain $S \leftrightarrow X \leftrightarrow Y$, and then share $Y$.

This encoding, characterized by the conditional probability distribution $P_{Y|X}$, is designed so that the representation $Y$ keeps a certain level $r$ of the information about the data $X$ that is not shared with the private attributes $S$ (i.e., the light gray area in Figure \ref{fig:i_diagram_cpf}), while minimizing the information it keeps about the private attributes $S$ (i.e., the dark gray area in Figure \ref{fig:i_diagram_cpf}). That is,
\begin{equation}
	\smash{\arg \inf_{P_{Y|X}} \left \lbrace I(S;Y) \right \rbrace \textnormal{ s.t. } I(X;Y|S) \geq r.}
	\label{eq:conditional_privacy_funnel_opt}
\end{equation}
This formulation reassembles the privacy funnel (PF)~\cite{makhdoumi2014information}, since both minimize the information the representation $Y$ keeps about the private attributes $S$. Nonetheless, in the PF, the encoding is designed so that the representation $Y$ keeps a certain level $r'$ of information about the data $X$, disregarding if this information is also shared by the private attributes $S$. Hence, the optimization of the PF Lagrangian may lead to representations $Y$ that filter private information arbitrarily. In contrast, the CPF avoids this issue with an inherent additional constraint on the Lagrange multipliers (see Appendix~\ref{app:lagrangians_pf_cpf_equivalence}).
\subsubsection{Fairness: the conditional fairness bottleneck (CFB)}
\label{subsubsec:fairness_formulation}
Suppose we want to use (or share) some data $X \in \cX$ to draw inferences or make decisions about a task $T \in \cT$. However, these data $X$ and the task $T$ contain information about some sensitive attributes $S \in \cS$, and we do not want our inferences to be influenced by these sensitive attributes $S$. For this reason, we encode the data $X$ into a representation $Y \in \cY$ and then use $Y$ to draw inferences about the task $T$. Therefore, the Markov chains $S \leftrightarrow X \rightarrow Y$ and $T \leftrightarrow X \rightarrow Y$ hold.

This encoding, characterized by the conditional probability distribution $P_{Y|X}$, is designed so that the representation $Y$ keeps a certain level $r$ of the information about the task $T$ that is not shared by the sensitive attributes $S$ (i.e., the light gray area in Figure \ref{fig:i_diagram_cib}), while minimizing the information it keeps about the sensitive attributes $S$ and the information about the data $X$ that is not shared with the task $T$ (i.e., the dark and darker gray areas in Figure \ref{fig:i_diagram_cib}, respectively). That is, 
\begin{equation}
    \smash{\arg \inf_{P_{Y|X}} \left \lbrace I(S;Y) + I(X;Y|S,T) \right \rbrace \textnormal{ s.t. } I(T;Y|S) \geq r.}
    \label{eq:conditional_fairness_bottleneck_opt}
\end{equation}
This formulation differs from other approaches to fairness in two main points: (i)~Similarly to the IB \cite{tishby2000information}, the CFB does not only minimize the information the representation keeps about the sensitive attributes $S$, but also the information about the data $X$ that is irrelevant for the task $T$. That is, the CFB seeks a representation that is both \emph{fair} and \emph{relevant}, thus avoiding the risk of keeping \emph{nuisances} \cite{achille2018emergence} and harming its generalization capability. (ii)~Similarly to the CEB \cite{fischer2020conditional}, the CFB aims to produce a representation $Y$ that maintains a certain level $r$ of the information about the task $T$ that is not shared by the sensitive attributes $S$. This differs from formulations that aim to keep a certain level $r'$ of the information about the task $T$, disregarding if it is also shared by the sensitive attributes $S$.
%
\begin{figure}[htbp]
    \centering
    \subfloat[Conditional privacy funnel.\label{fig:i_diagram_cpf}]{{\includegraphics[width=0.25\textwidth]{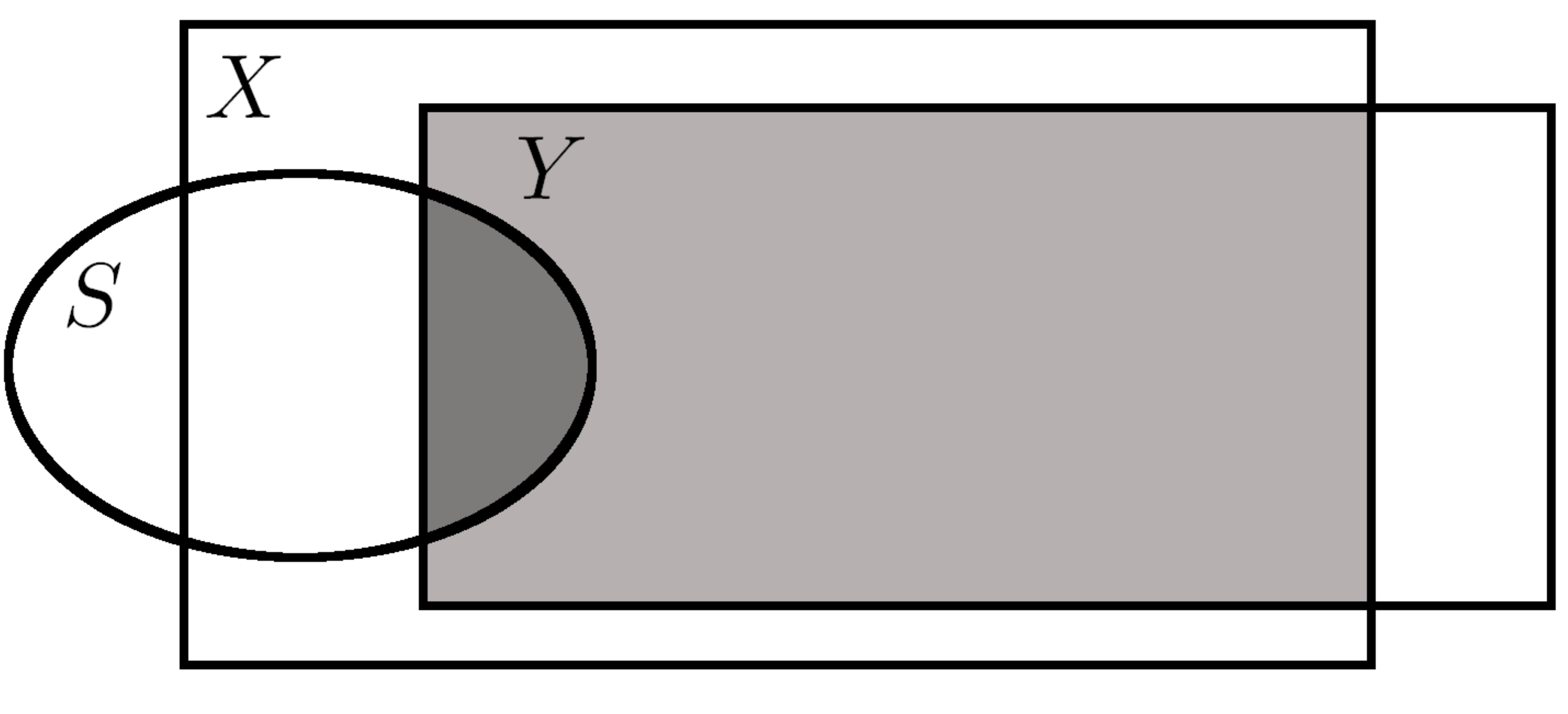}}} 
    \subfloat[Conditional fairness bottleneck. \label{fig:i_diagram_cib}]{{\includegraphics[width=0.25\textwidth]{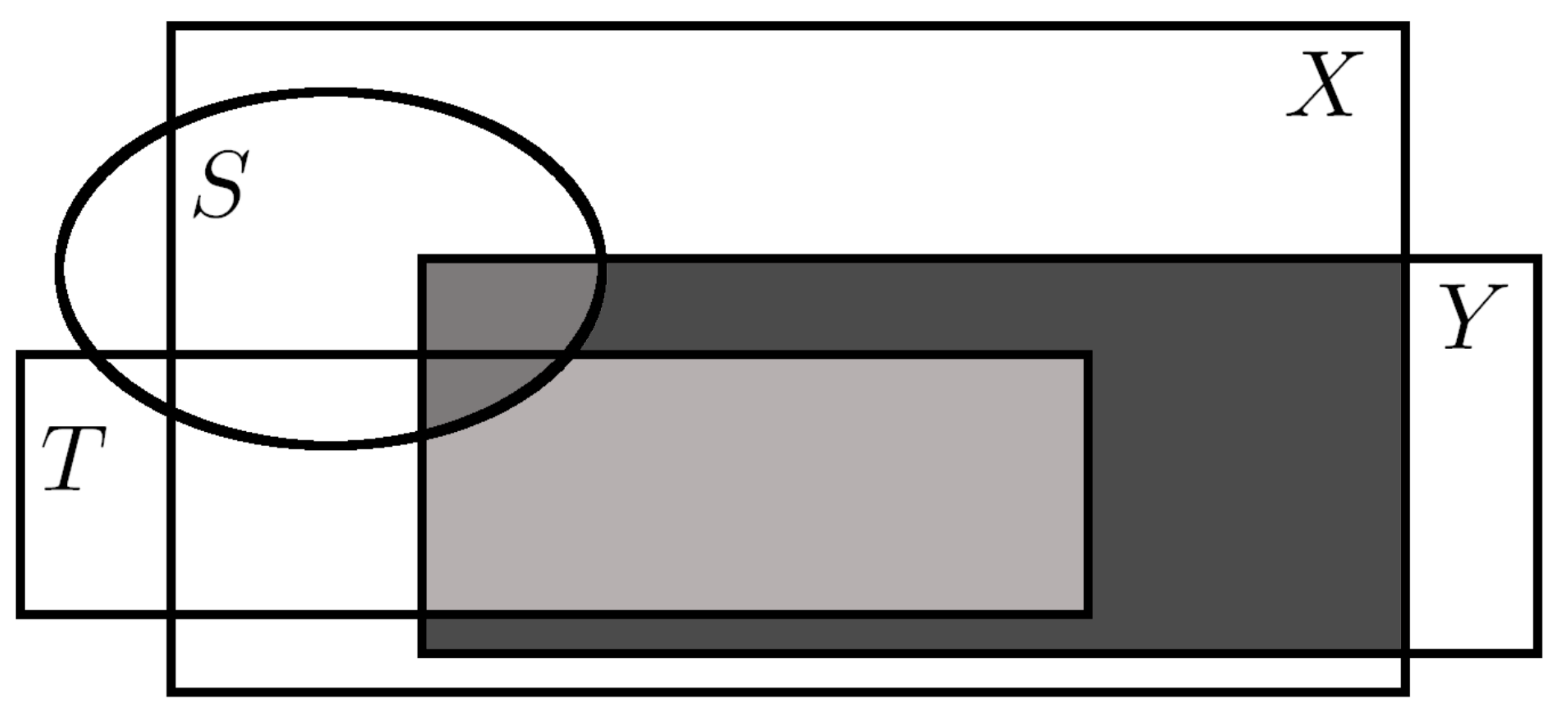}}}
    \caption{Information Diagrams \cite{yeung1991new} of (a) the CPF and (b) the CFB. In light gray, the relevant information we want $Y$ to keep. In dark and darker gray, the sensitive and irrelevant information, respectively, we want $Y$ to discard.}
    \label{fig:i_diagram}
\end{figure}
\subsection{The Lagrangians of the problems}
\label{subsec:lagrangians}
A common approach to solving optimization problems such as the CPF or the CFB is to minimize the \emph{Lagrangian} of the problem. The Lagrangian is a proxy of the trade-off between the function to optimize and the constraints on the optimization search space \cite[Chapter 5]{boyd2004convex}. Particularly, the Lagrangians of the CPF and the CFB are, ommiting the constant term $\lambda r$ in the optimization w.r.t. $P_{Y|X}$, respectively, 
\begin{align*}
    \smash{\mathcal{L}_{\textnormal{CPF}}(P_{Y|X}, \lambda)} &\smash{= I(S;Y) - \lambda I(X;Y|S) \textnormal{ and}} \\
    \smash{\mathcal{L}_{\textnormal{CFB}}(P_{Y|X}, \lambda)} &\smash{= I(S;Y) + I(X;Y|S,T) - \lambda I(T;Y|S),} 
\end{align*}
where $\lambda \geq 0 \in \mathbb{R}$ is the \emph{Lagrange multiplier} of the Lagrangian. This multiplier controls the trade-off between the information the representation $Y$ discards and the information it keeps.\footnote{Note that if $\lambda = 0$ the optimization only seeks for maximally compressed representations $Y$. Hence, trivial encoding distributions like a degenerate distribution $P_{Y|X}$ with density $p_{Y|X} = \delta(Y)$ minimize the Lagrangian.} More specifically, $\lambda$ controls exactly the trade-off between the information we want the representation to keep (i.e., the light gray area from Figure~\ref{fig:i_diagram}) and all the other information (i.e., the dark gray area from Figure~\ref{fig:i_diagram}).

In the following propositions, proved in Appendix \ref{app:lagrangians_variational_equivalence}, we present two alternative Lagrangians that are equivalent to the original Lagrangians, are more tractable, and exhibit similar properties and structure in the privacy and fairness problems.
\begin{restatable}{proposition}{equivalencelagrangianscpf} 
Minimizing $\mathcal{L}_{\textnormal{CPF}}(P_{Y|X}, \lambda)$ is equivalent to minimizing $\mathcal{J}_{\textnormal{CPF}}(P_{Y|X},\gamma)$, where $\gamma = \lambda + 1$ and  
\begin{equation}
	\smash{\mathcal{J}_{\textnormal{CPF}}(P_{Y|X},\gamma) = I(X;Y) - \gamma I(X;Y|S).}
\end{equation}
\label{prop:equivalence_lagrangians_cpf}
\end{restatable}
\begin{restatable}{proposition}{equivalencelagrangianscfb} 
Minimizing $\mathcal{L}_{\textnormal{CFB}}(P_{Y|X}, \lambda)$ is equivalent to minimizing $\mathcal{J}_{\textnormal{CFB}}(P_{Y|X},\beta)$, where $\beta = \lambda + 1$ and 
\begin{equation}
	\smash{\mathcal{J}_{\textnormal{CFB}}(P_{Y|X},\beta) = I(X;Y) - \beta I(T;Y|S).}
\end{equation}
\label{prop:equivalence_lagrangians_cfb}
\end{restatable}
The minimization of $\mathcal{J}_{\textnormal{CPF}}$ and $\mathcal{J}_{\textnormal{CFB}}$, by means of $\gamma$ and $\beta$, trades off the level of compression of the representation Y with the information 
it keeps, respectively, about the data $X$ and the task $T$ that is not shared by the sensitive attributes $S$.
\subsection{The variational approach}
\label{subsec:variational_approach}
We consider the minimization of $\mathcal{J}_{\textnormal{CPF}}$ and $\mathcal{J}_{\textnormal{CFB}}$ to solve the CPF and CFB problems. Furthermore, we assume that the probability density (or mass if $|\mathcal{Y}|$ is countable) functions $p_{Y|X}$ that describe the conditional probability distribution $P_{Y|X}$  exist and are parameterized by $\theta$, i.e., $p_{Y|X} = p_{Y|X,\theta}$.

The Markov chains of the CPF and the CFB characterize the densities
$p_{S,X,Y|\theta} = p_{S,X} \cdot p_{Y|X,\theta}$ and $p_{S,T,X,Y|\theta} = p_{S,T,X} \cdot p_{Y|X,\theta}$. The densities $p_{S,X}$ and $p_{S,T,X}$ can be inferred from the data and the density $p_{Y|X,\theta}$ is to be designed.

The term $I(X;Y)$ depends on the density $p_{Y|\theta}$, which is usually intractable. Similarly, the terms $I(X;Y|S)$ and $I(T;Y|S)$ depend on the densities $p_{X|S,Y,\theta}$ and $p_{T|S,Y,\theta}$, respectively, which are also usually intractable. Therefore, an exact optimization of $\theta$ is prohibitively computationally expensive. For this reason, we introduce the variational density approximations $q_{Y|\theta}$, $q_{X|S,Y,\phi}$, and $q_{T|S,Y,\phi}$, where the generative and inference densities are parametrized by $\phi$. 

Then, as previously done in, e.g., \cite{kingma2013auto, alemi2016deep, kolchinsky2019nonlinear}, we leverage the non-negativity of the relative entropy 
to bound $\mathcal{J}_{\textnormal{CPF}}$ and $\mathcal{J}_{\textnormal{CFB}}$ from above. More precisely, we bound $I(X;Y)$ from above and $I(X;Y|S)$ and $I(T;Y|S)$ from below, i.e.,
%
%
\begin{align}
    I(X;Y) 
    &\leq \mathbb{E}_{p_X} \left[  D_{\textnormal{KL}} \left( p_{Y|X,\theta} || q_{Y|\theta}) \right) \right], \label{eq:bound_ixy} \\
    I(X;Y|S) 
    &\geq \mathbb{E}_{p_{S,X,Y|\theta}} \left[ \log \left( \frac{q_{X|S,Y,\phi}}{p_{X|S}} \right) \right] \textnormal{, and} \label{eq:bound_ixy_s} \\
    I(T;Y|S) 
    &\geq \mathbb{E}_{p_{S,T,Y|\theta}} \left[ \log \left( \frac{q_{T|S,Y,\phi}}{p_{T|S}} \right) \right]. \label{eq:bound_iyt_s}
\end{align}
Finally, we can jointly learn $\theta$ and $\phi$ through gradient descent. First, we note that the terms $\mathbb{E}_{p_{(S,X)}}[\log p_{X|S}]$ and $\mathbb{E}_{p_{(S,T)}}[\log p_{T|S}]$ do not depend on the parametrization, and can therefore be discarded from the optimization. Second, we leverage the \emph{reparametrization trick} \cite{kingma2013auto}, which allows us to compute an unbiased estimate of the gradients. That is, we let $p_{Y|X}dY = p_E dE$, where $E$ is a random variable and $Y = f(X,E;\theta)$ is a deterministic function. Lastly, we estimate $p_{S,X}$ and $p_{S,T,X}$ as the data's empirical densities.

In practice, if we have a dataset of $N$ samples $\left(x^{(i)}, s^{(i)}\right)$ for the CPF or $N$ samples $\left(x^{(i)}, s^{(i)}, t^{(i)}\right)$ for the CFB, we minimize, respectively, the following cost functions:
\begin{align}
    \tilde{\mathcal{J}}_{\textnormal{CPF}}&(\theta,\phi,\gamma) = \frac{1}{N}  \sum\nolimits_{i=1}^N  D_{\textnormal{KL}} \left( p_{Y|X=x^{(i)},\theta} || q_{Y|\theta} \right) \nonumber \\ 
    &- \gamma \mathbb{E}_{p_E} \left[ \log \left( q_{X|S=s^{(i)},Y=f(x^{(i)},E), \phi} (x^{(i)})\right)\right] \label{eq:cpf_cost}
\end{align}
\begin{align}
    \tilde{\mathcal{J}}_{\textnormal{CFB}}&(\theta,\phi,\beta) = \frac{1}{N} \sum\nolimits_{i=1}^N  D_{\textnormal{KL}} \left( p_{Y|X=x^{(i)},\theta} || q_{Y|\theta} \right) \nonumber \\
    &- \beta \mathbb{E}_{p_E} \left[ \log \left( q_{T|S=s^{(i)},Y=f(x^{(i)},E), \phi} (t^{(i)})\right)\right], \label{eq:cfb_cost}
\end{align}
where the expectation over $E \in \mathcal{E}$ is usually
estimated with a naive Monte Carlo of a single sample.

An \emph{a posteriori} interpretation of this approach is that if the encoder compresses the representation $Y$ assuming that the decoder will use both $Y$ and the private or sensitive attributes $S$, then the encoder will discard the information about $S$ contained in the original data $X$ in order to generate $Y$.
\begin{remark}
\label{remark:easily_adopted}
The resulting cost functions for the CPF and the CFB ressemble those of the VAE \cite{kingma2013auto}, the $\beta$-VAE \cite{higgins2017beta}, the VIB \cite{alemi2016deep}, or the nonlinear IB \cite{kolchinsky2019nonlinear}. Consider the (common) case that the decoder density is estimated with a neural network. If such a network is modified so that it receives as input both the representation and the private or sensitive attributes instead of
only the representation, then the optimization of these algorithms results in private and/or fair representations (see Appendix~\ref{app:modification_common_algorithms} for the details).
\end{remark}
%

\section{Results}
\label{sec:results}

In this section, we present experiments on two datasets to showcase the performance of the presented variational approach to the privacy and fairness problems. First, we show its performance in a dataset commonly used for benchmarking both tasks. Second, we show the performance on high-dimensional data on a toy dataset designed for this purpose. The encoder density is modeled with an isotropic Gaussian distribution, i.e.,  $p_{Y|X,\theta} = \mathcal{N}(Y;\mu_{\textnormal{enc}}(X;\theta), \sigma_{\theta}^2 I_d)$, 
where $\mu_{\textnormal{enc}}$ is a neural network and $d$ is the dimension of the representation. The marginal density of the representation 
is also modeled as an isotropic Gaussian $q_{Y|\theta} = \mathcal{N}(Y;0,I_d)$. Finally, the decoder density, $q_{X|S,Y,\phi}$ or $q_{T|S,Y,\phi}$, 
is modeled with a product of categorical (for discrete data) and/or isotropic Gaussians (for continuous data), e.g., $q_{X|S,Y,\phi} = \textnormal{Cat}(X_1;\rho_{\textnormal{dec}}(Y,S;\phi)) \mathcal{N}(X_2;\mu_{\textnormal{dec}}(S,Y;\phi), \sigma_{\phi}^2)$ if $X$ 
consists of a discrete variable $X_1$ and a continuous variable $X_2$. These and additional experiments are detailed in Appendix \ref{app:experiments}. 
\paragraph{Adult dataset} The Adult dataset (available at the UCI machine learning repository \cite{Dua2019}) contains $45,222$ samples from the 1994 U.S. Census. Each sample comprises 15 features such as, e.g., \emph{gender}, \emph{age}, or \emph{income level} (binary variable stating if the income level is higher than $\$ 50,000$). We followed the experimental set-up from \cite{zemel2013learning}, where the sensitive attribute $S$ is the gender, the task $T$ is the income level, and the data $X$ is the rest of the features. 
\paragraph{Toy dataset (Colored MNIST)} The MNIST dataset \cite{lecun1998gradient} is a collection of $70,000$ grayscale $28 \times 28$ images of hand-written digits from $0$ to $9$. The Colored MNIST is a modification of the former dataset where each digit is randomly colored in either red, green, or blue. We considered that the data $X$ are the $3 \times 28 \times 28$ digit images, the sensitive attribute $S$ is the digit's color, and the task $T$ is digit's number.
\subsection{Privacy}
\label{subsec:results_privacy}
The proposed approach is able to control the trade-off between a private and an informative representation for both the Adult and the Colored MNIST datasets. We minimized~\eqref{eq:cpf_cost} for different values of $\gamma \in [1,50]$, thus controlling the trade-off between
the compression $I(X;Y)$ and the informativeness of the representation independent of the private data $I(X;Y|S)$ (see Figures~\ref{fig:adult_ixy_vs_minus_hx_given_sy} and~\ref{fig:colored_mnist_ixy_vs_minus_hx_given_sy}). Hence, as suggested by Proposition \ref{prop:equivalence_lagrangians_cpf}, 
the multiplier $\gamma$ also controls the amount of private information the representation keeps (see Figures \ref{fig:adult_gammas_vs_isy} and \ref{fig:colored_gammas_vs_isy}).
\begin{figure}[htbp]
    \centering
    \subfloat[Privacy on Adult.\label{fig:adult_ixy_vs_minus_hx_given_sy}]{{\includegraphics[width=0.23\textwidth]{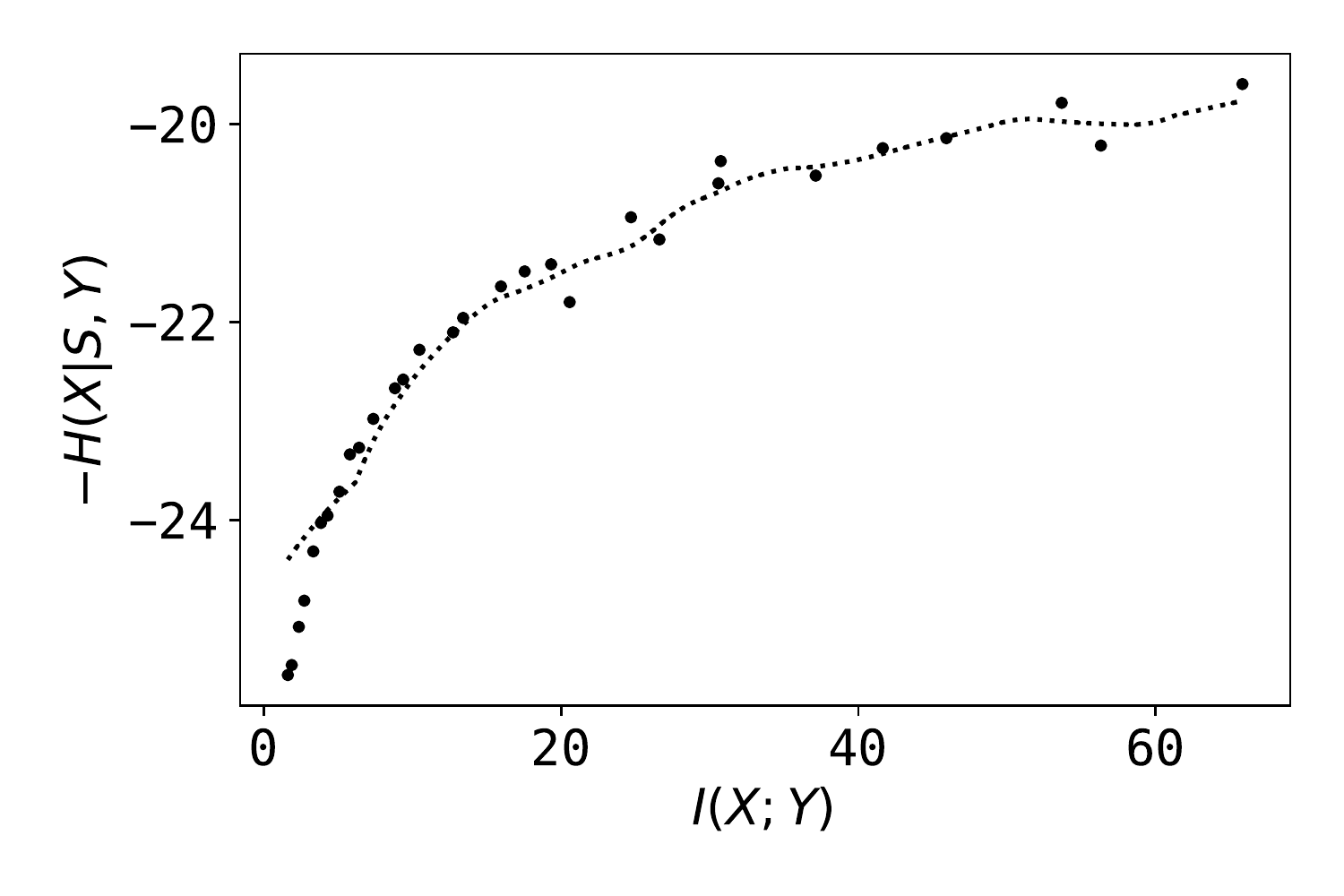}}}  
    \subfloat[Privacy on Colored MNIST.\label{fig:colored_mnist_ixy_vs_minus_hx_given_sy}]{{\includegraphics[width=0.23\textwidth]{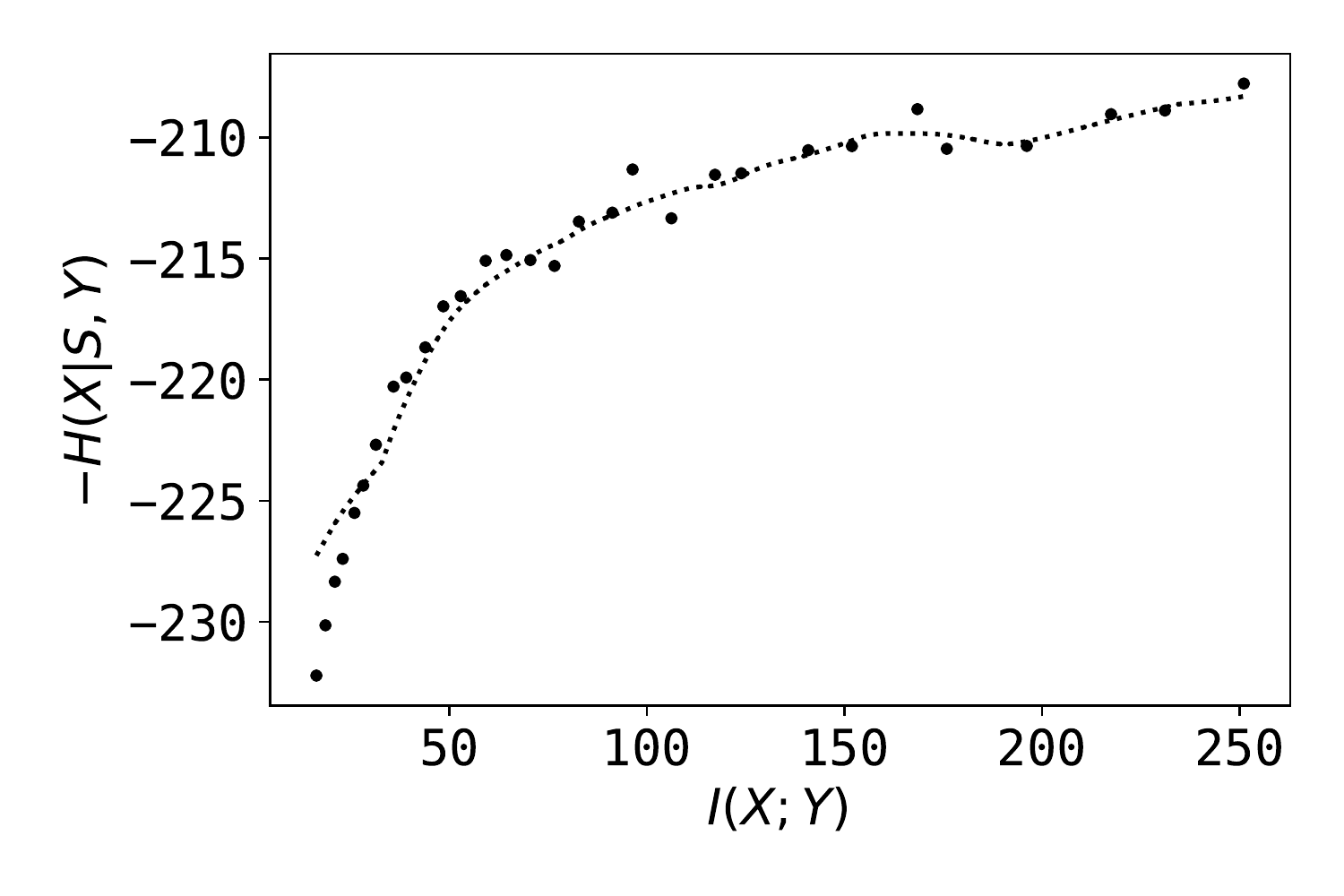}}}
    
    \subfloat[Fairness on Adult. \label{fig:adult_ixy_vs_ity_given_s}]{{\includegraphics[width=0.23\textwidth]{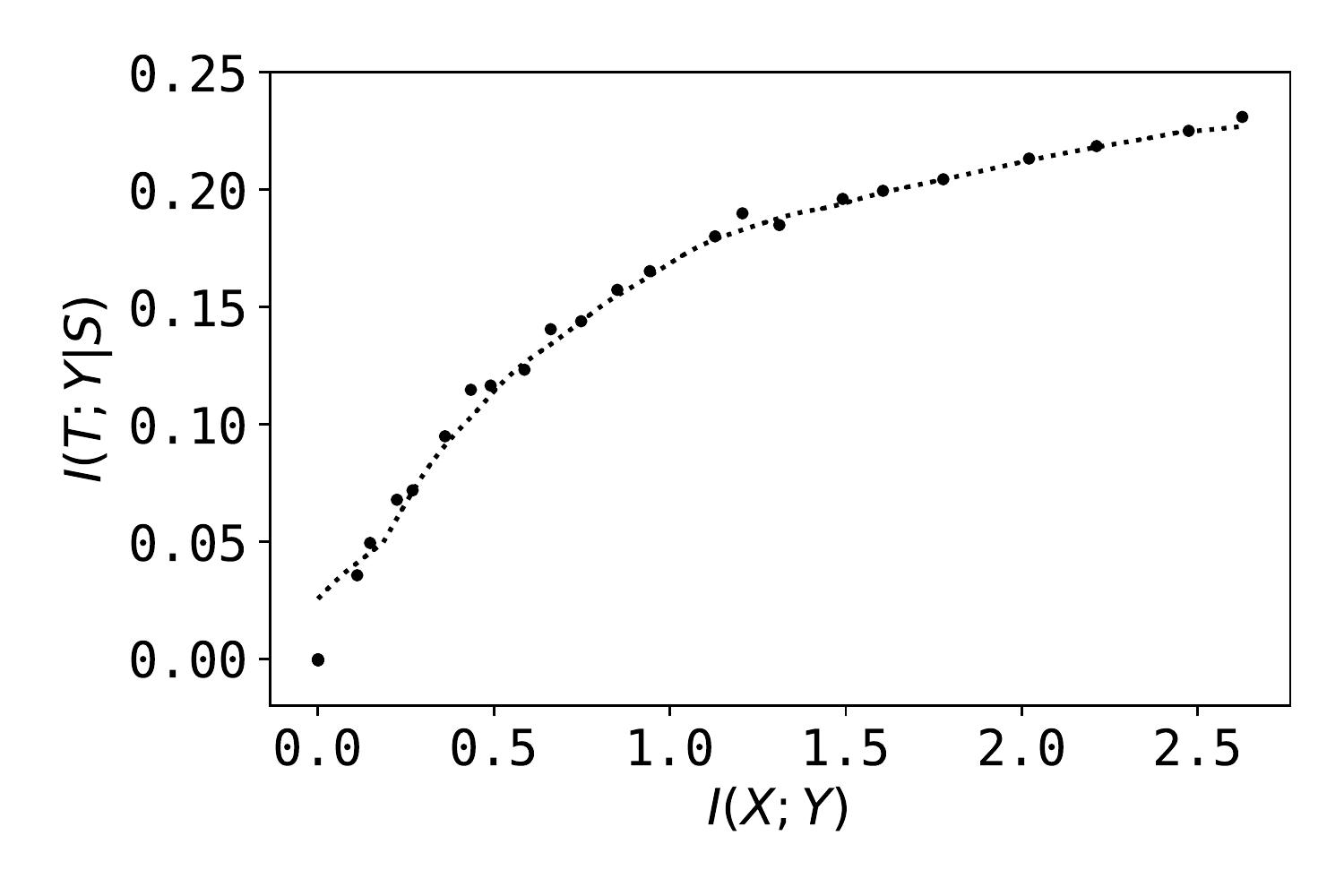}}}
    \subfloat[Fairness on Colored MNIST. \label{fig:colored_mnist_ixy_vs_ity_given_s}]{{\includegraphics[width=0.23\textwidth]{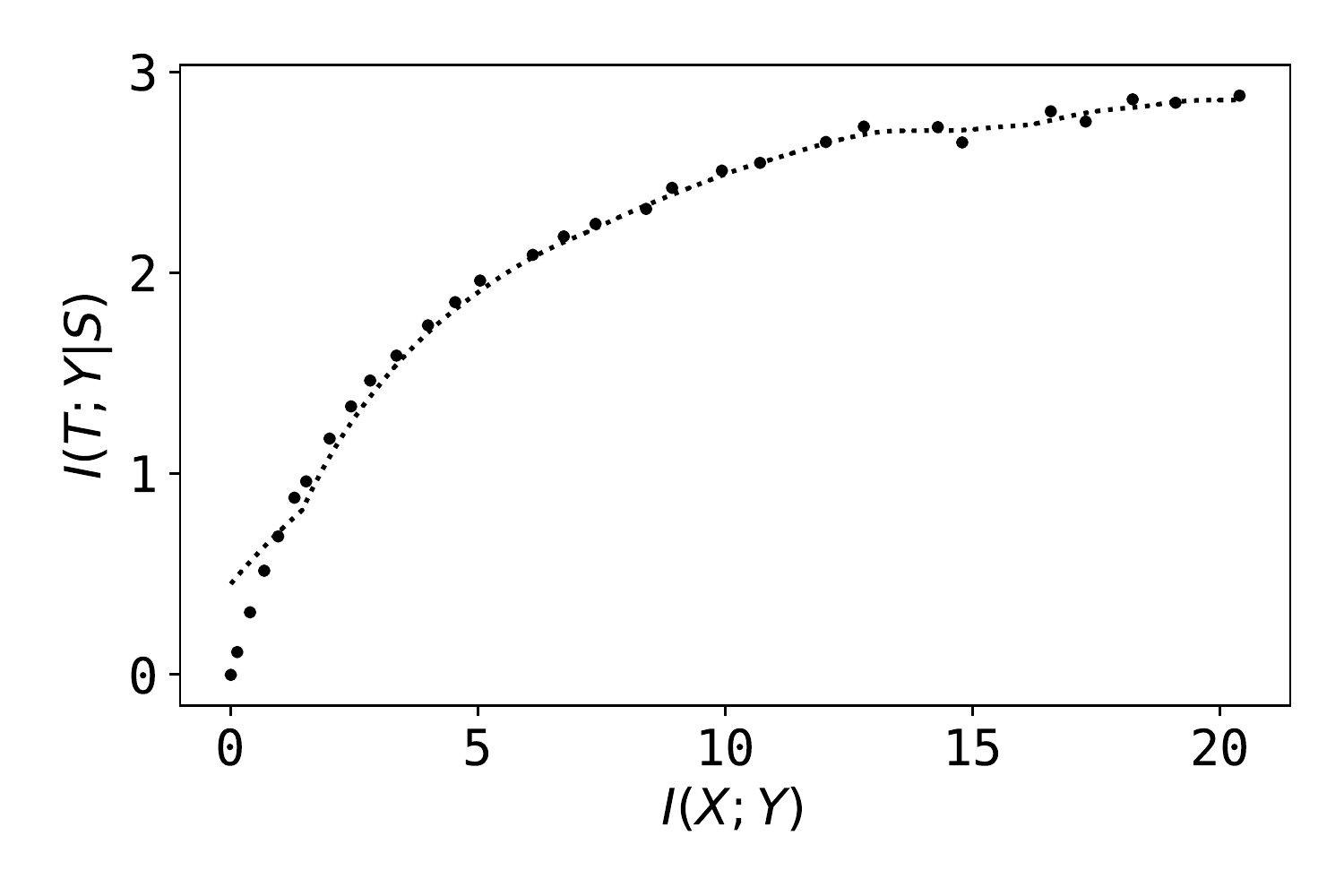}}}
    
    \subfloat[Privacy on Adult. \label{fig:adult_gammas_vs_isy}]{{\includegraphics[width=0.23\textwidth]{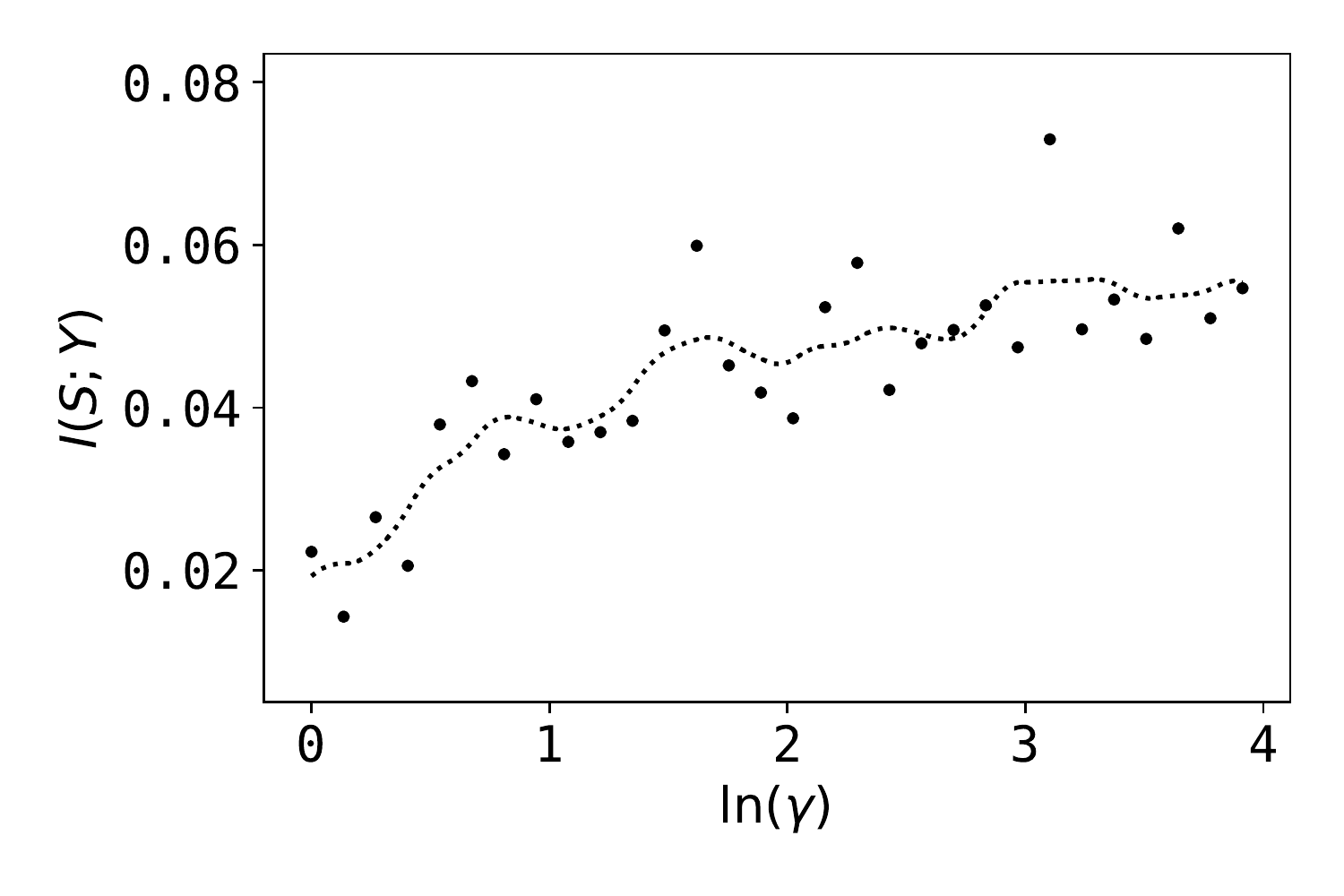}}} 
    \subfloat[Privacy on Colored MNSIT. \label{fig:colored_gammas_vs_isy}]{{\includegraphics[width=0.23\textwidth]{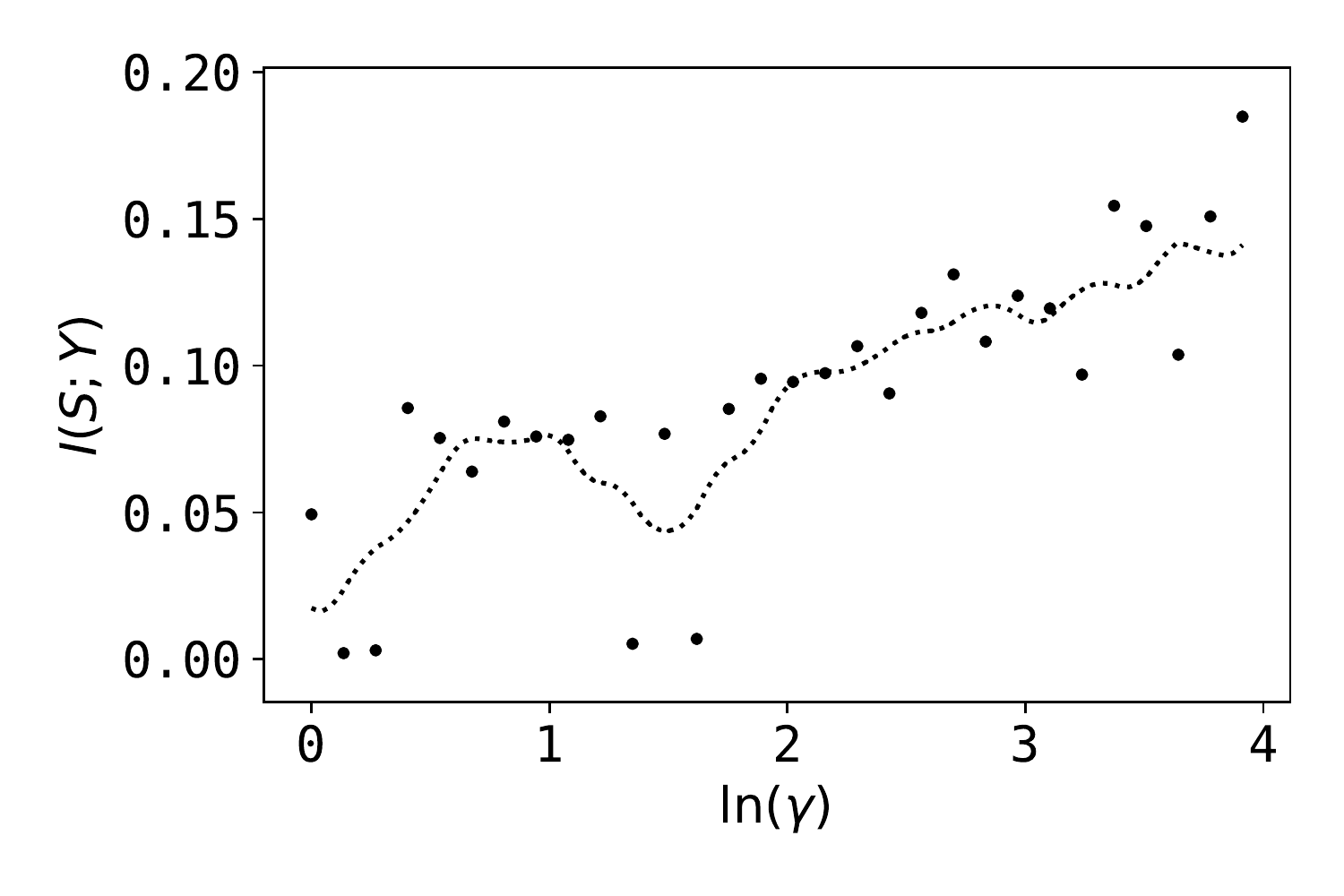}}}
    
    \subfloat[Fairness on Adult. \label{fig:adult_betas_vs_isy}]{{\includegraphics[width=0.23\textwidth]{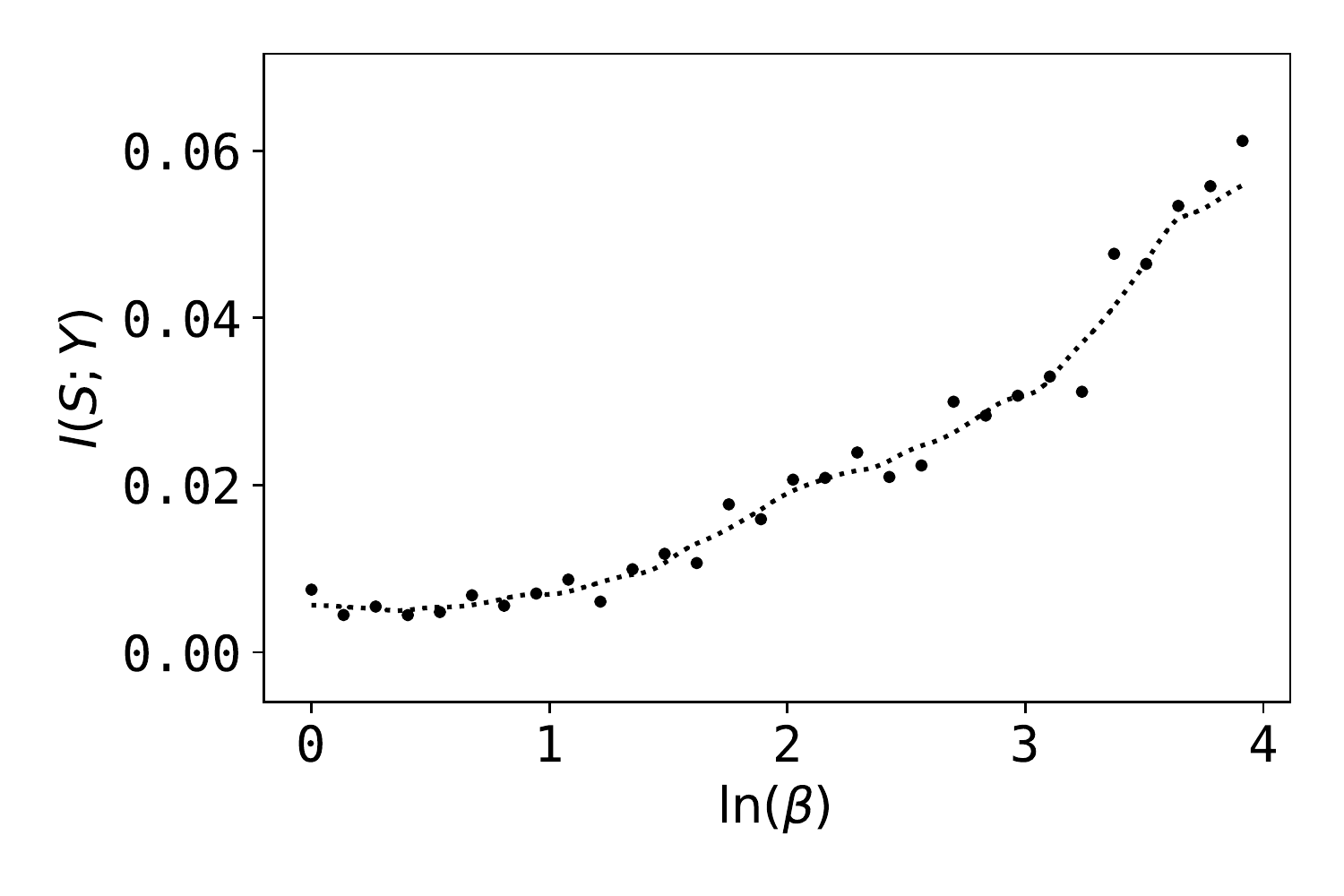}}}
    \subfloat[Fairness on Colored MNIST. \label{fig:colored_mnist_betas_vs_isy}]{{\includegraphics[width=0.23\textwidth]{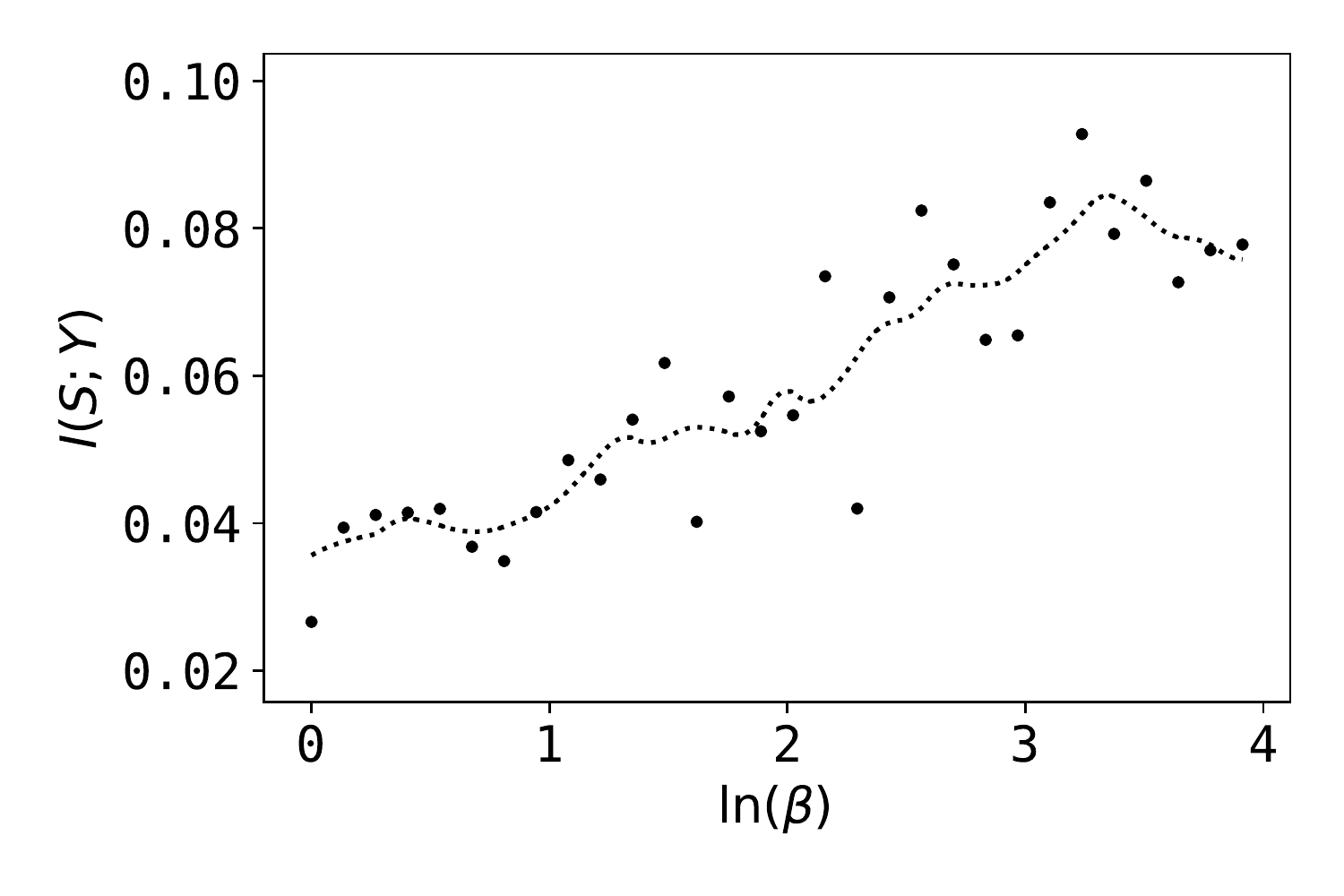}}}
    \caption{
    First two rows: representation's compression $I(X;Y)$ vs. non-private information retained about the data $I(X;Y|S)$ or the task $I(T;Y|S)$. Since $I(X;Y|S) = H(X|X) - H(X|S,Y)$ and $H(X|S)$ does not depend on $Y$, only $-H(X|S,Y)$ is reported. Last two rows: private information leaked $I(S;Y)$ (estimated with MINE~\cite{belghazi2018mutual}) vs. different values of the Lagrange multipliers $\gamma \in [1,50]$ and $\beta \in [1,50]$. The dots are the computed empirical values and the lines the moving average of their linear interpolation.
    }
\end{figure}

As an illustration, we constructed a representation
with the same dimension to the digit images by minimizing~\eqref{eq:cpf_cost} with $\gamma = 1$ (Figure~\ref{fig:private_representation}). The representation is
both informative and private; e.g., the 2D UMAP \cite{mcinnes2018umap} vectors of the representation are mingled with respect to the digits' color, as opposed to the UMAP vectors of the original images, where the vectors are clustered by the color of the digits (see Figures~\ref{fig:x_reduction} and~\ref{fig:y_reduction}). 
\begin{figure}[htbp]
    \centering
    \subfloat[Original's data $X$ samples. \label{fig:data_representation}]{{\includegraphics[width=0.2\textwidth]{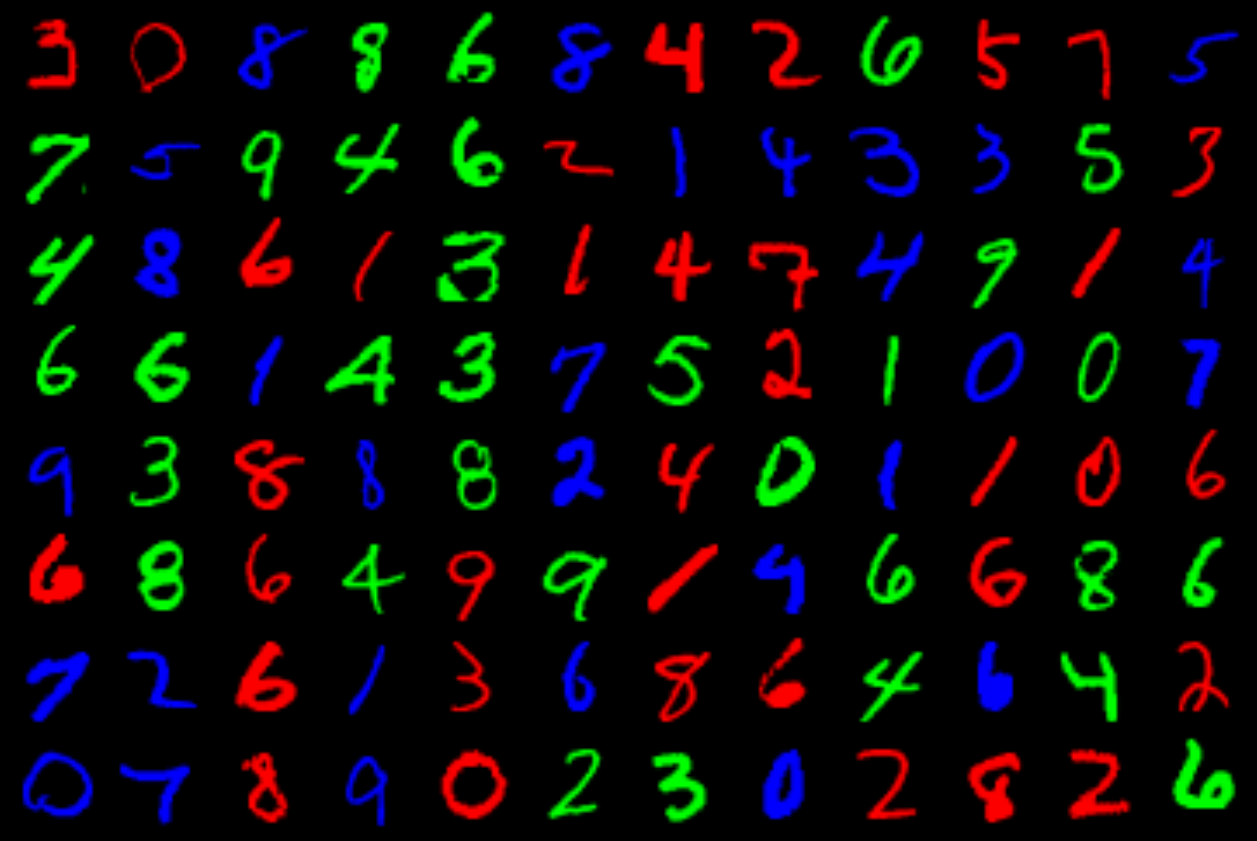}}} \
    \subfloat[Representation's $Y$ samples. \label{fig:private_representation}]{{\includegraphics[width=0.2\textwidth]{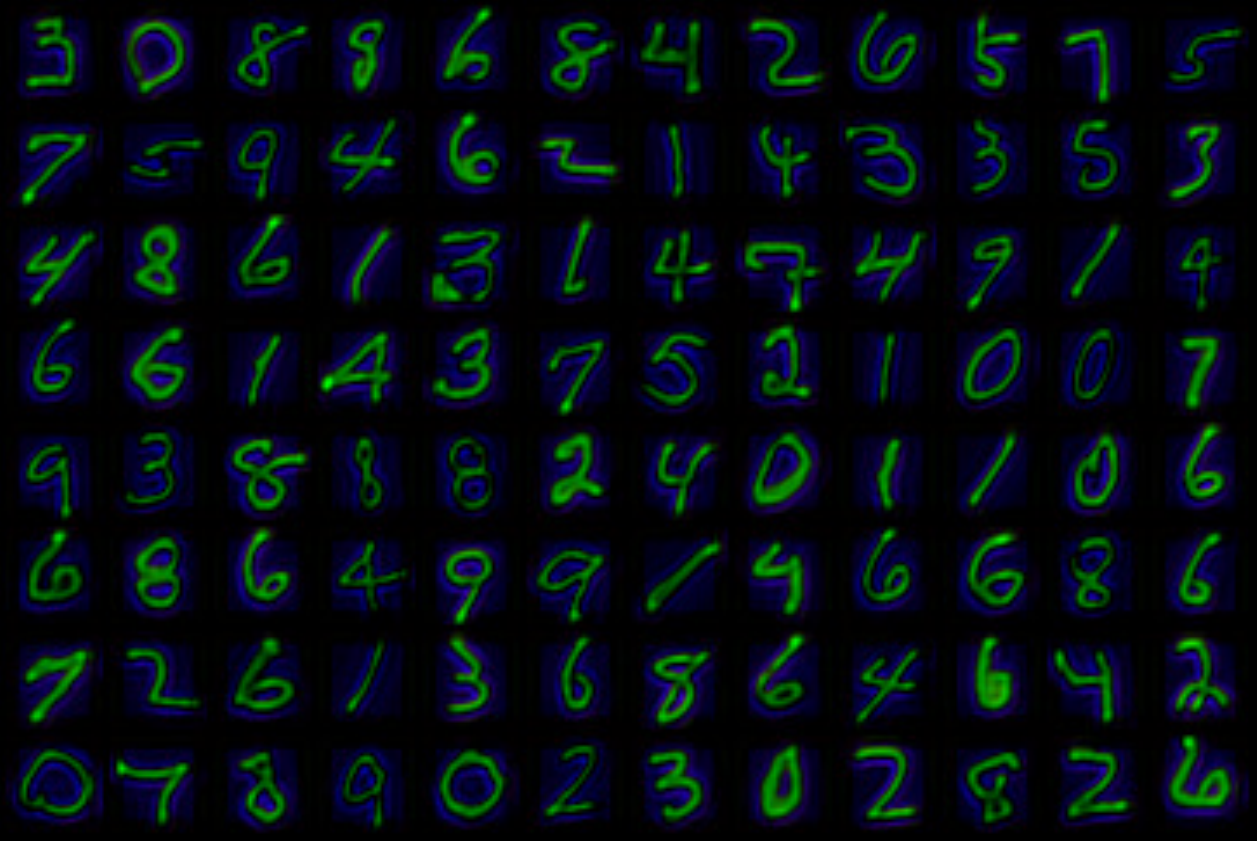}}} 
    
    \subfloat[Original's data $X$ UMAP. \label{fig:x_reduction}]{{\includegraphics[width=0.2\textwidth]{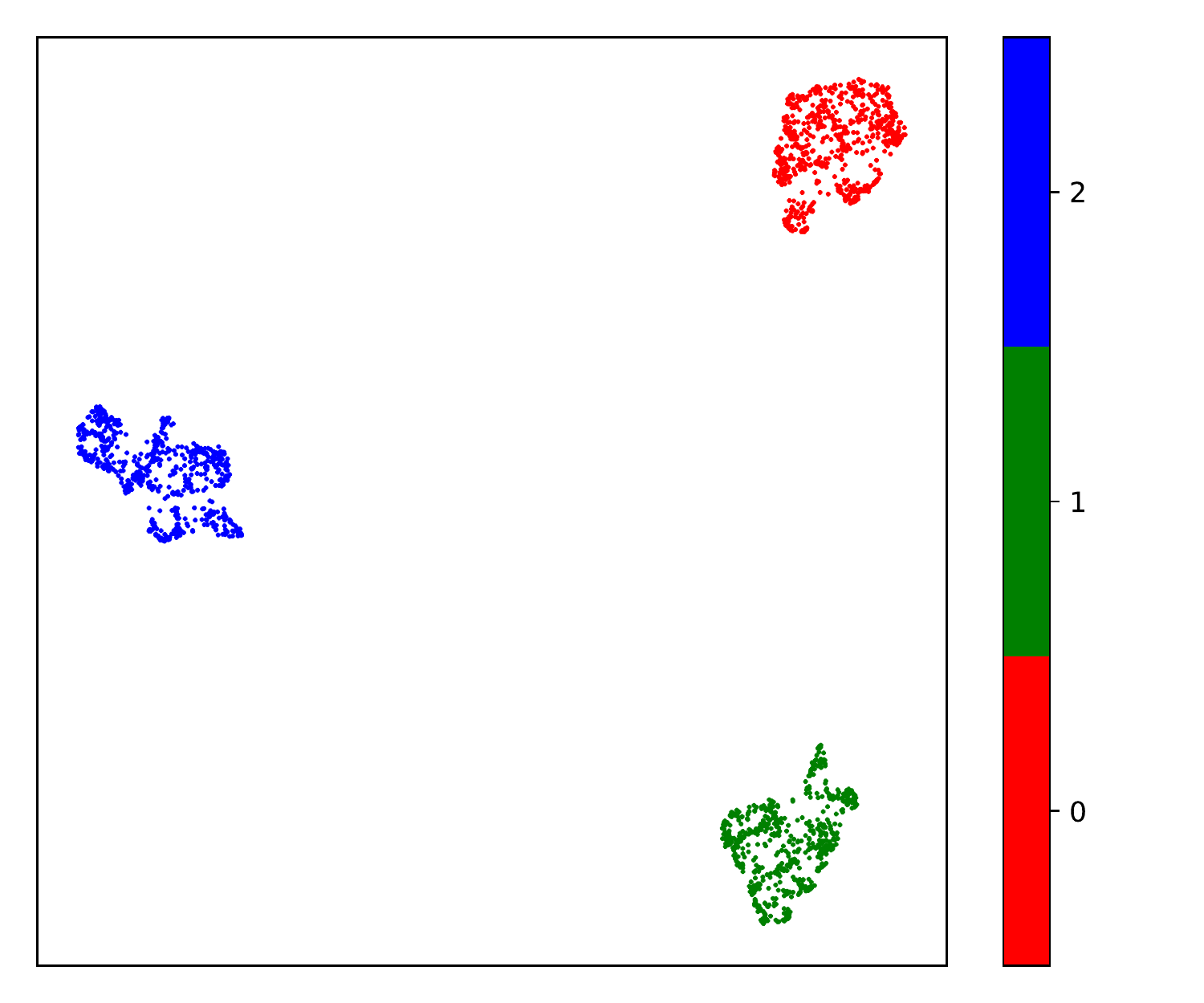}}}
    \subfloat[Representation's $Y$ UMAP. \label{fig:y_reduction}]{{\includegraphics[width=0.2\textwidth]{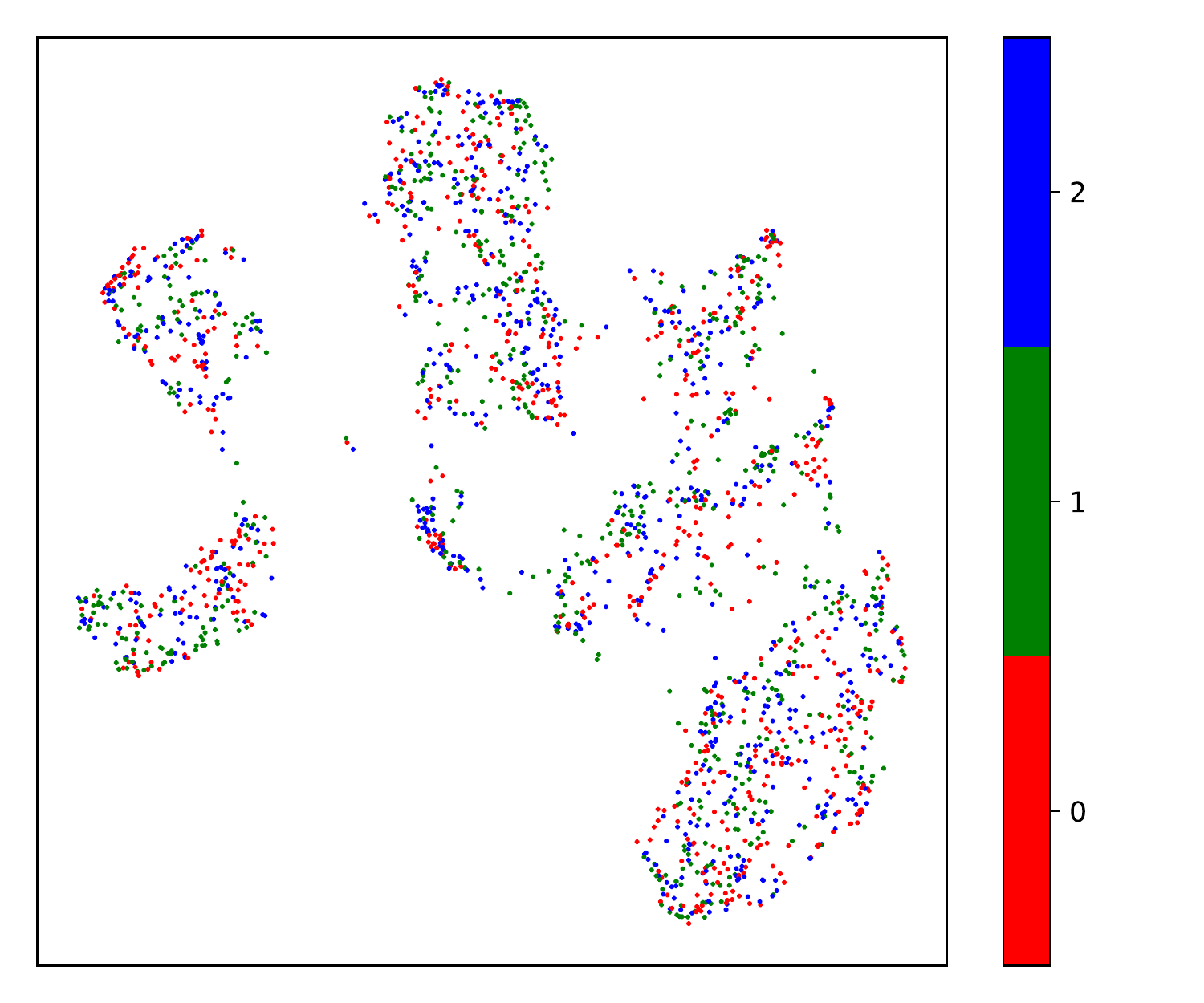}}}
    \caption{Random samples of Colored MNIST and their UMAP dimensionality reduction vectors. Each UMAP vector is colored with the same color to the digit they represent. The results are obtained for $\gamma = 1$.}
    \label{fig:example_cpf}
\end{figure}

Compared to other variational approaches to the PF like the \cite[PPVAE]{nan2020variational} and \cite[VFAE]{louizos2015variational}, the proposed approach performed better in terms of information leakage $I(S;Y)$ and accuracy of non-linear attackers 
(see Table~\ref{table:privacy_adult}). More specifically, the leaked information is about an order of magnitude lower than the obtained in the other methods and the accuracy of a random forest attacker is below the prior probability (0.67), while the other methods allow successful attacks of 0.8 -- 0.9 accuracy.
An explanation for this phenomenon is that some private information is leaked to the representations $Y$ of the PPVAE and the VFAE via their encoding density $p_{Y|(S,X,\theta)}$, which does not respect the Markov chain $S \leftrightarrow X \rightarrow Y$. 
\begin{table}[htbp]
 \centering
 \caption{Random forest attacker's accuracy and private information leaked on the Adult dataset. Parameter ranges of the PPVAE and VFAE: $\eta^{-1} \in [1,50]$ and $\delta \in [N_{\textnormal{batch}},1000N_{\textnormal{batch}}]$.}
 \begin{tabular}{c|c|c}
    \toprule
    Methods & Attacker's accuracy (S) & $I(S;Y)$ \\
    \midrule
    Ours & \textbf{0.60 -- 0.64} & \textbf{0.01 -- 0.08} \\
    PPVAE  & 0.79 -- 0.93 & 0.29 -- 0.63 \\
    VFAE & 0.81 -- 0.95 & 0.28 -- 0.44 \\
    \bottomrule
 \end{tabular}
 \label{table:privacy_adult}
\end{table}
\subsection{Fairness}
\label{subsec:results_fairness}
\begin{table*}[htbp!]
 \centering
 \caption{{Fairness metrics with different methods on Adult dataset. Displayed as: Logistic regression / Random forest. The best hyperparameters for the other methods have been selected (see Appendix \ref{app:experiments}}). These models and compared to our model with the hyperparameters that reaches the most similar performance.}
 \label{table:comparison_other_methods}
 \begin{tabular}{c|c|c|c|c|c}
    \toprule
    Methods & Accuracy (T) & Accuracy (S) & Discrimination & Error gap & EO gap \\
    \midrule
    LFR \cite{zemel2013learning} & 0.84 / 0.84 & 0.66 / \textbf{0.98} & 0.07 / \textbf{0.16} & 0.10 / 0.12 & \textbf{0.18} / 0.07 \\
    Ours ($\beta = 17.0$) & 0.82 / 0.80 & 0.66 / 0.62 & 0.08 / 0.08 & 0.12 / 0.13 & 0.09 / 0.09\\
    \midrule 
    FFVAE \cite[$\alpha=200$]{creager2019flexibly} &0.77 / 0.75 & 0.67 / 0.61 & 0.00 / 0.01 & 0.17 / 0.12 & 0.04 / 0.04 \\
    CFAIR \cite[$\lambda=1000$]{zhao2019conditional} & 0.77 / 0.71 & 0.69 / 0.65 & 0.02 / 0.01 & 0.17 / 0.12 & 0.02 / 0.04 \\
    Ours ($\beta = 1.96$) & 0.76 / 0.72 & 0.67 / 0.61 & 0.00 / 0.00 & 0.19 / 0.15 & 0.00 / 0.00\\
    \bottomrule
 \end{tabular}
\end{table*}

\begin{figure}[htbp]
	\centering
    \subfloat[Task accuracy. \label{fig:betas_vs_accuracy}]{{\includegraphics[width=0.23\textwidth]{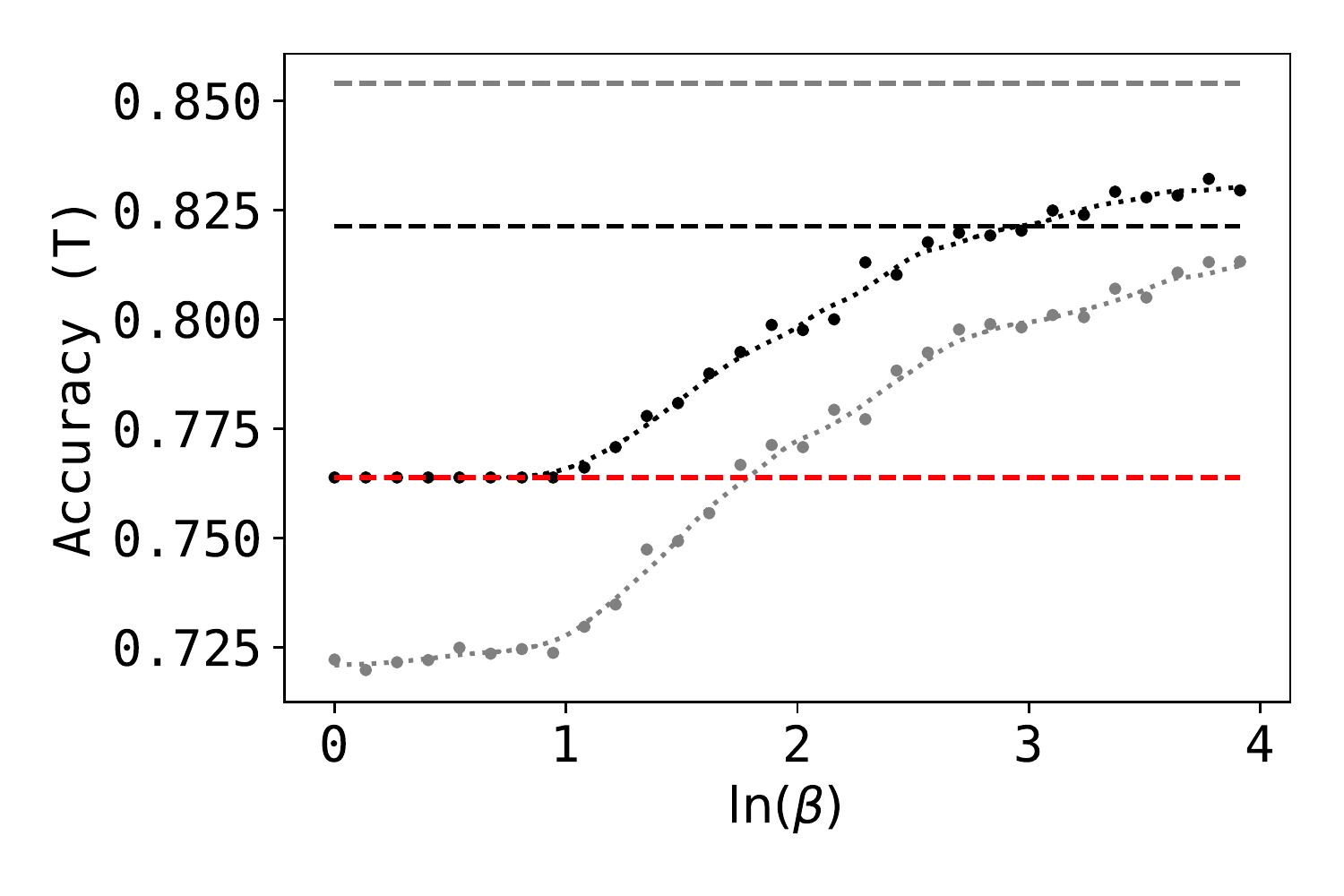}}}
    \subfloat[Discrimination.\label{fig:betas_vs_discrimination}]{{\includegraphics[width=0.23\textwidth]{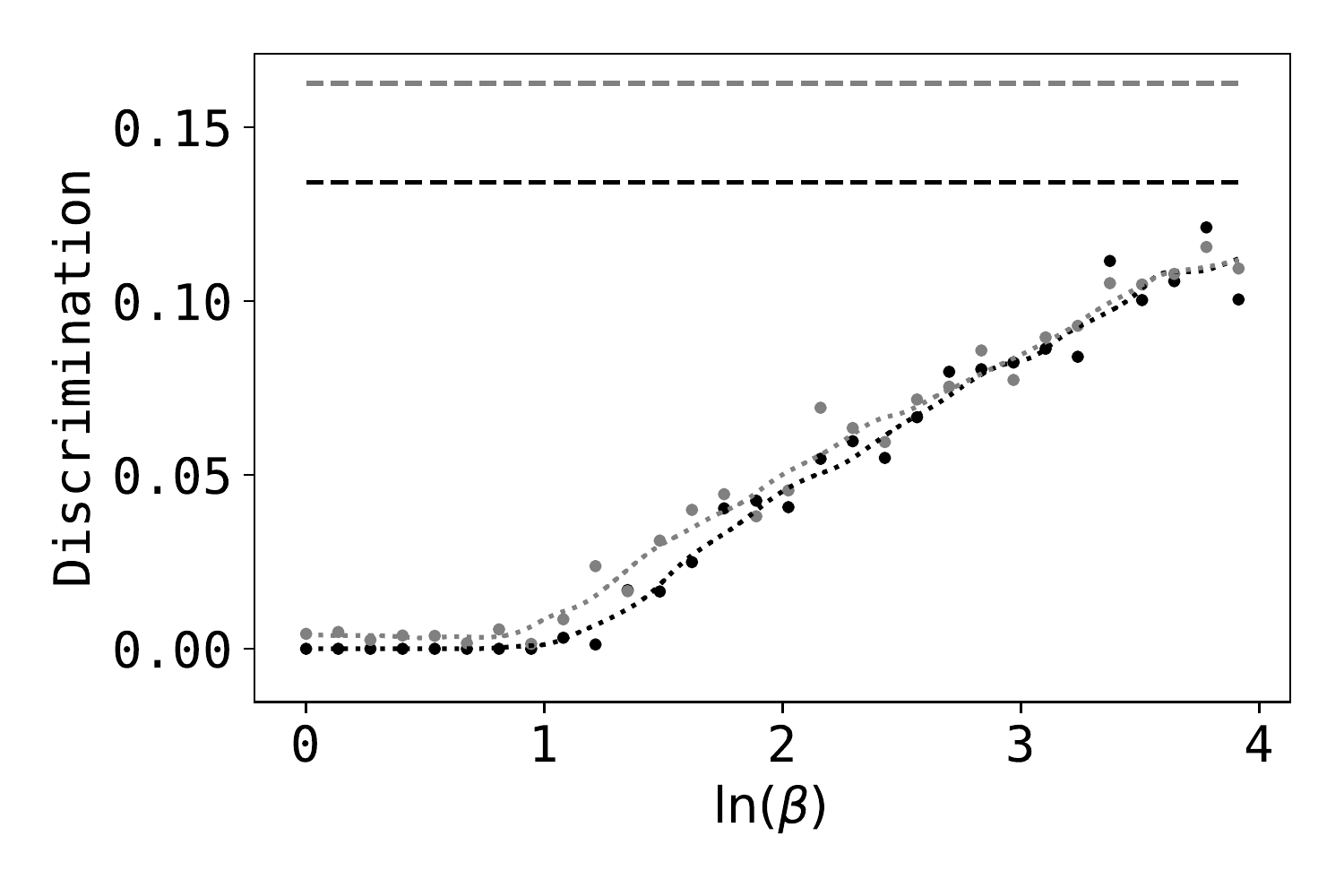}}}
    
    \subfloat[Equalized odds gap. \label{fig:betas_vs_eq_odds}]{{\includegraphics[width=0.23\textwidth]{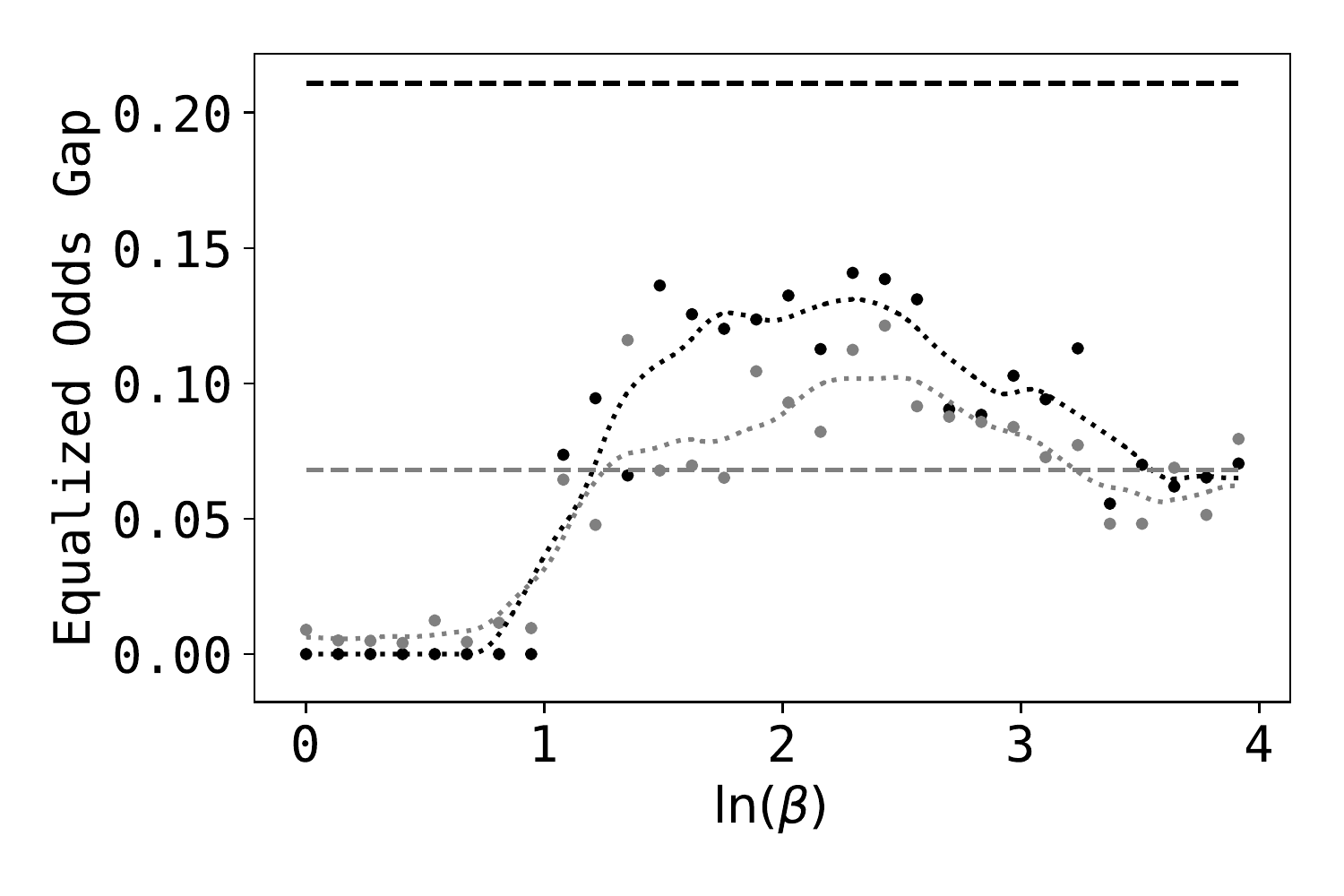}}}
    \subfloat[Error gap. \label{fig:betas_vs_err_gap}]{{\includegraphics[width=0.23\textwidth]{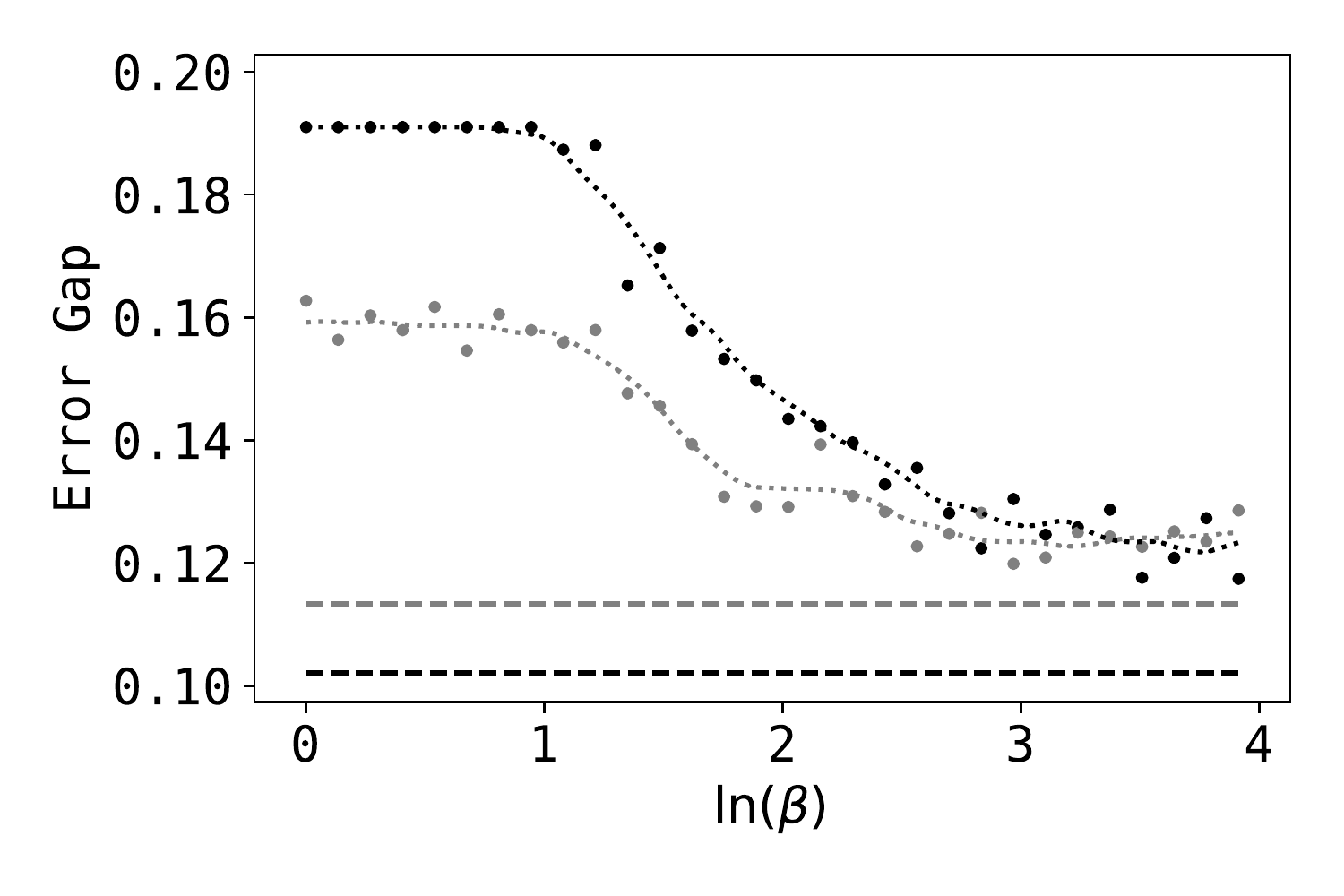}}}
    
    \subfloat[Sensitive data accuracy. \label{fig:betas_vs_accuracy_s}]{{\includegraphics[width=0.23\textwidth]{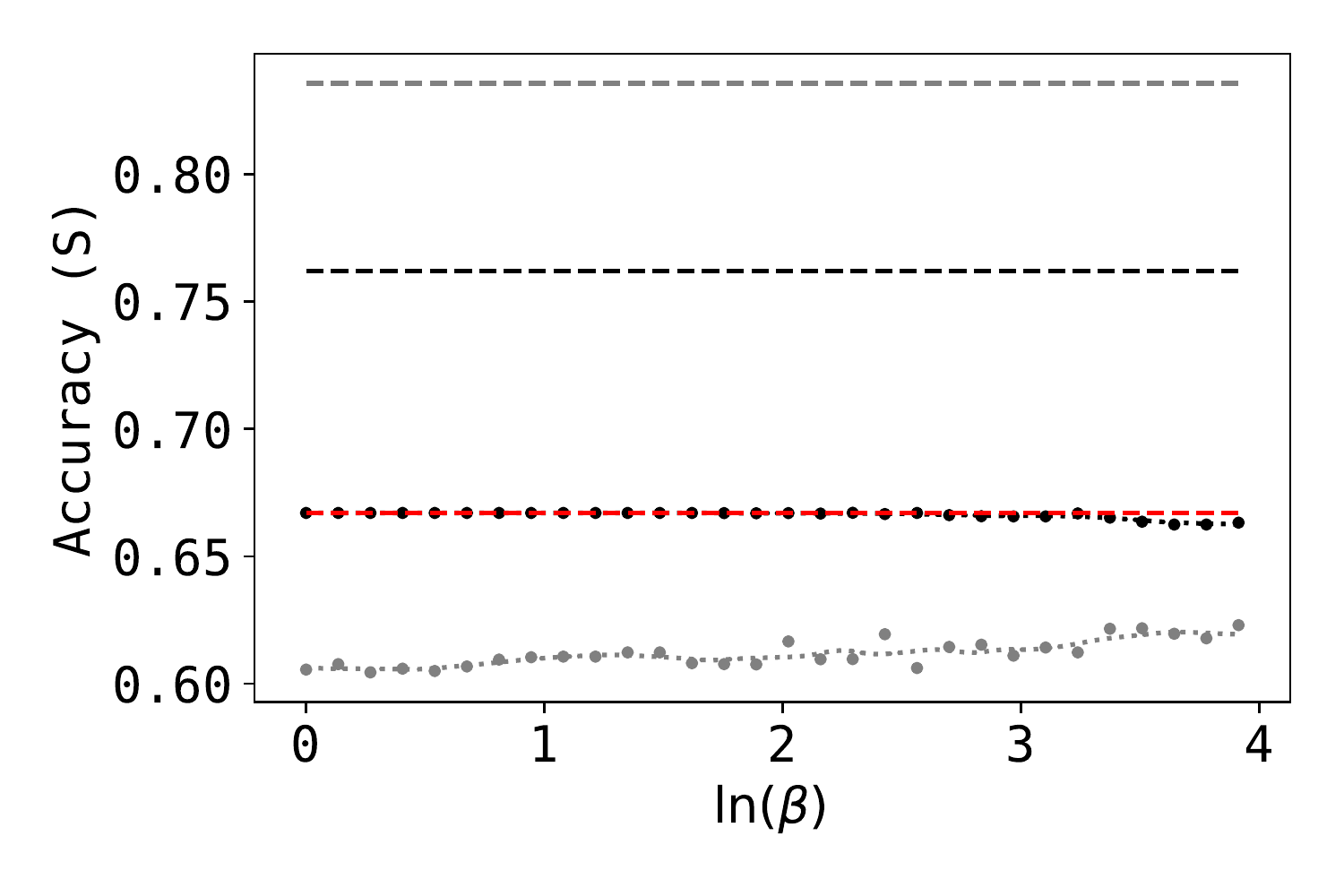}}} 
    \subfloat{{\includegraphics[width=0.17\textwidth]{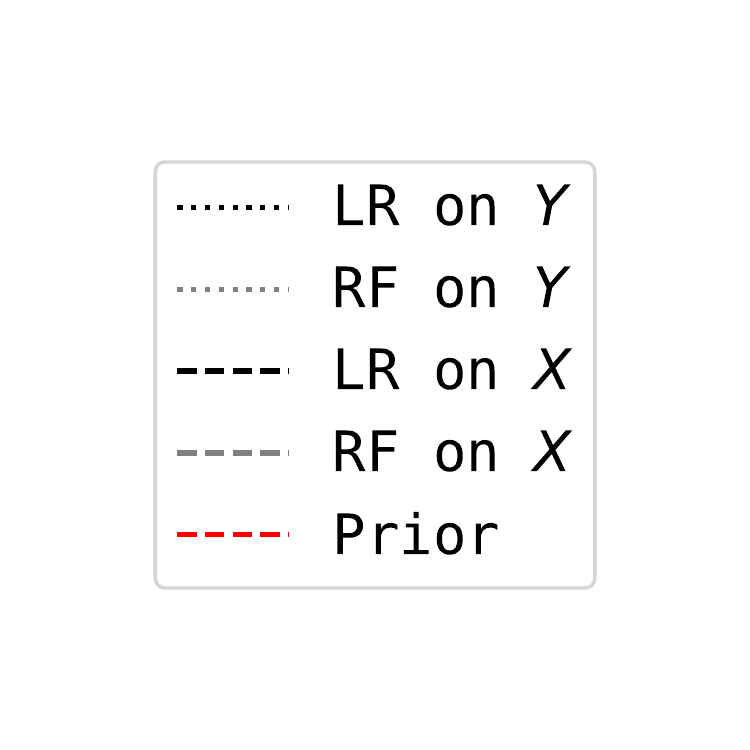}}}
    \caption{
    Behavior of different utility and group fairness indicators of a logistic regression (LR, in black) and a random forest (RF, in gray) of the fair representations (dots and dotted line) and the original data (dashed line) learned with $\beta \in [1,50]$ on the Adult dataset. The dashed line in red is the accuracy of a prior-based classifier.
    }
\end{figure}

The proposed variational approach is also able to control the trade-off between fair and accurate representations. We minimized~\eqref{eq:cfb_cost} for different values of $\beta \in [1,50]$, thus controlling the trade-off between the compression $I(X;Y)$ and the informativeness of the representations independent of the sensitive attributes $I(T;Y|S)$ (see Figures~\ref{fig:adult_ixy_vs_ity_given_s} and~\ref{fig:colored_mnist_ixy_vs_ity_given_s}). Thus, as hinted by Proposition~\ref{prop:equivalence_lagrangians_cfb}, the multiplier $\beta$ also controls the amount of 
private information leaked (see Figures~\ref{fig:adult_betas_vs_isy} and~\ref{fig:colored_mnist_betas_vs_isy}). 

Furthermore, in the Adult dataset, the Lagrange multiplier $\beta$ allows us to control the behavior of different utility and group fairness indicators (defined in Appendix~\ref{app:fairness_indicators}), namely the accuracy, the error gap, and the discrimination (or demographic parity gap). That is, the higher the value of $\beta$, the higher the accuracy and the discrimination, and the lower the error gap (see Figures~\ref{fig:betas_vs_accuracy}, \ref{fig:betas_vs_discrimination}, and~\ref{fig:betas_vs_err_gap}). The behavior of the discrimination is enforced by the minimization of $I(S;Y)$, as discussed in Remark~\ref{remark:cfb_and_cpf_encourage_demographic_parity} from Appendix~\ref{app:fairness_indicators}. However, there is no clear indication of the effect of $\beta$ on the accuracy of adversarial predictors on the sensitive data (which is still below the prior probability of the biased training dataset) and on the equalized odds (see Figures~\ref{fig:betas_vs_accuracy_s} and~\ref{fig:betas_vs_eq_odds}). The equalized odds are not optimized in this scenario since the resulting representation $Y$ is such that $I(S;Y;T) < 0$ as discussed in Remark~\ref{remark:cfb_and_cpf_equalized_odds} from Appendix~\ref{app:fairness_indicators}. An example where the equalized odds gap is minimized is presented in Appendix~\ref{app:experiments}.

The representation generated from the proposed method is more robust to non-linear adversaries (as shown in Table \ref{table:comparison_other_methods} by the smaller accuracy on $S$ and smaller discrimination for random forest adversaries) than the standard baseline from \cite[LFR]{zemel2013learning}. Compared with other state-of-the-art methods, either those employing variational inference \cite[FFVAE]{creager2019flexibly} or adversarial training \cite[CFAIR]{zhao2019conditional}, the proposed method for generating fair representations achieves similar results (see Table~\ref{table:comparison_other_methods}), albeit having an easier cost function and requiring to train less networks, respectively (see Appendix~\ref{app:related_work}). 

\section{Discussion}
\label{sec:discussion}

In this article, we studied the problem of mitigating private or sensitive information $S$ leakage into data-driven systems through the training data $X$. We formalized the trade-off between the relevant information for the system that is not shared by the private or sensitive attributes $S$ and the remaining information as a constrained optimization problem. When the task of the system is unknown, the problem is referred as learning private representations and the formalization is the conditional privacy funnel (CPF); and when it is kwnon the problem is referred as learning fair representations and the formalization is the conditional fairness bottleneck (CFB).

To solve these problems, we proposed a variational approach based on the Lagrangians of the CPF and the CFB. This approach leads to a simple structure highlighting the similarities between private and fair representation learning. Moreover, in practice, private and fair representations can be learned with little modification to the implementation of common algorithms such as the VAE \cite{kingma2013auto}, the $\smash{\beta}$-VAE \cite{higgins2017beta}, the VIB \cite{alemi2016deep}, or the nonlinear IB \cite{kolchinsky2019nonlinear}. Namely, modifying the decoder neural network so it receives both the representation $Y$ and the private or sensitive attributes $S$ as an input. Then, the learned representations can be fed to any algorithm of choice. For this reason, the efforts for reducing unfair decisions and privacy breaches will be small for many practitioners.

The proposed formulation has some limitations due to the non-convexity of the CPF and the CFB (see Appendix~\ref{app:convexity_problems}). Moreover, the proposed approach also faces limitations due to its variational nature. In Appendix~\ref{app:limitations}, we reflect on these limitations and propose future directions to overcome them.


\IEEEtriggeratref{55}
\bibliographystyle{IEEEtran}
\bibliography{references}

\newpage
\onecolumn
\appendices
\section{Motivation of mutual information as utility and privacy/fairness metric}
\label{app:why-mi}

The main reason to choose mutual information (and conditional mutual information) as the measure of utility and privacy and/or fairness is the tractability of this metric and the fact that it allowed us to (i) draw connections between the privacy and fairness problems and (ii) derive an algorithm that could be easily incorporated to current approaches to representation learning.

Other than that, even though there are some caveats of using mutual information as a measure of privacy (see \cite{issa2019operational}) or fairness, this metric has several operational meanings. Namely:

\begin{itemize}
    \item \textbf{Utility.} The conditional mutual information $I(X;Y|S)$ (or $I(T;Y|S)$) is a measure of the relevance of the representation $Y$ to explain $X$ (or $T$). This is in line with the information bottleneck \cite{tishby2000information, alemi2016deep, kolchinsky2019nonlinear} and other common representation learning algorithms \cite{hjelm2018learning}, \cite{makhdoumi2014information}. 
        
    An intuition would be that if we want to keep the average distortion smaller than a certain quantity, and we consider the distortion to be the log-loss, then we want that the (conditional) mutual information is greater than a, different, certain quantity (see e.g. \cite[Section III-B]{makhdoumi2014information}).
    
    \item \textbf{Privacy.} The mutual information $I(S;Y)$ is the average cost gain by an adversary with the self-information or log-loss cost function \cite[Lemma 1]{makhdoumi2014information}. This has also been used in other works such as \cite{chatzikokolakis2010statistical, zhu2005anonymity}. Moreover, an upper bound on mutual information limits both the Bayesian and minimax risks that can be achieved in any inference based on the data; see e.g. \cite[Chapter 15.3 on Fano's method]{MW}.
    
    \item \textbf{Fairness.} As explained in Remark~\ref{remark:fairness_indicators_representations} minimizing $I(S;Y)$ encourages demographic parity and minimizing $I(S;Y|T)$ encourages equalized odds.
\end{itemize}

%
\section{Equivalences of the Lagrangians}
In this section of the appendix, we show how minimizing the Lagrangians of the CPF and the CFB problems is equivalent to minimizing other Lagrangians. First, in \ref{app:lagrangians_pf_cpf_equivalence} we show that minimizing the Lagrangian of the CPF is equivalent to minimizing the Lagrangian of the PF, meaning that the conditional probability distributions $P_{Y|X}$ obtained using the Lagrangian of the CPF could have been obtained through the Lagrangian of the PF, too. Then, in \ref{app:lagrangians_variational_equivalence} we show that minimizing the CPF and CFB Lagrangians is equivalent to minimizing the Lagrangians that are used in the variational approach we propose in this paper.
\subsection{Equivalence of the Lagrangians of the privacy funnel and the CPF}
\label{app:lagrangians_pf_cpf_equivalence}
The privacy funnel is defined in a similar way to the CPF. It is an optimization problem that tries to design an encoding probability distribution $P_{Y|X}$ such that the representation $Y$ keeps a certain level $r'$ of information about the data of interest $X$, while minimizing the information it keeps about the private data $S$ \cite{makhdoumi2014information}. That is, 
\begin{equation*}
    \arg \inf_{P_{Y|X}} \lbrace I(S;Y) \rbrace \textnormal{ s.t. } I(X;Y) \geq r'.
\end{equation*}
Therefore, the Lagrangian of the privacy funnel problem is
\begin{equation*}
    \mathcal{L}_{\textnormal{PF}}(P_{Y|X}, \alpha) = I(S;Y) - \alpha I(X;Y),
\end{equation*}
where $\alpha \geq 0$ is the Lagrange multiplier of $\mathcal{L}_{\textnormal{PF}}(P_{Y|X}, \alpha)$. This multiplier controls the trade-off between the information the representations keep about the private and the original data. If $\alpha=1$, then $\mathcal{L}_{\textnormal{PF}}(P_{Y|X},1) = - I(X;Y|S)$, for which optimal values of the encoding distribution $P_{Y|X}$ can filter private information arbitrarily. If $\alpha > 1$ this problem is even more pronounced. If $\alpha = 0$, trivial encoding distributions like a degenerate distribution $P_{Y|X}$ with density $p_{Y|X} = \delta(Y)$ are minimizers of the Lagrangian. Therefore, in practice one might need to restrict $\alpha$ to the range $(0,1)$.
\begin{proposition} 
\label{prop:eq_pf_cpf}
Minimizing $\mathcal{L}_{\textnormal{CPF}}(P_{Y|X}, \lambda)$ is equivalent to minimizing $\mathcal{L}_{\textnormal{PF}}(P_{Y|X}, \alpha)$, where $\alpha = \lambda / (\lambda + 1)$.
\end{proposition}
\begin{proof}
If we manipulate the expression of the CPF Lagrangian we can see how the minimizing $ \mathcal{L}_{\textnormal{CPF}}(P_{Y|X}, \lambda)$ is equivalent to minimizing $\mathcal{L}_{\textnormal{PF}}(P_{Y|X}, \alpha)$, where $\alpha = \lambda / (\lambda + 1)$. More specifically,

\begin{align*}
	\arg &\inf_{P_{Y|X} \in \mathcal{P}} \left \lbrace \mathcal{L}_{\textnormal{CPF}}(P_{Y|X}, \lambda) \right \rbrace \nonumber \\ 
	&= \arg \inf_{P_{Y|X} \in \mathcal{P}} \left \lbrace I(S;Y) - \lambda I(X;Y|S) \right \rbrace \\
	&=  \arg \inf_{P_{Y|X} \in \mathcal{P}} \left \lbrace I(S;Y) - \lambda \left( I(X;Y) - I(S;Y)\right) \right \rbrace \\
	&=  \arg \inf_{P_{Y|X} \in \mathcal{P}} \left \lbrace (\lambda + 1) I(S;Y) - \lambda I(X;Y) \right \rbrace \\&=  \arg \inf_{P_{Y|X} \in \mathcal{P}} \left\lbrace  (\lambda + 1) \left( I(S;Y) - \frac{\lambda}{\lambda + 1} I(X;Y) \right) \right \rbrace \\
	&= \arg \inf_{P_{Y|X} \in \mathcal{P}} \left \lbrace \mathcal{L}_{\textnormal{PF}}\left(P_{Y|X}, \frac{\lambda}{\lambda + 1}\right) \right \rbrace,
\end{align*}
where $\mathcal{P}$ is the set of probability distributions over $\mathcal{Y}$ such that if $P_{Y|X=x} \in \mathcal{P}$ for all $x \in \mathcal{X}$, then the Markov chain $S \leftrightarrow X \rightarrow Y$ holds.
\end{proof}
We note how the relationship $\alpha = \lambda / (\lambda + 1)$ maintains $\alpha \in [0,1)$ for $\lambda \geq 0$. This showcases how the CPF poses a more restrictive problem, in the sense that as long as $\lambda < \infty$ there are no solutions of the problem that filter private information arbitrarily.
\subsection{Equivalence of the Lagrangians used for the minimization}
\label{app:lagrangians_variational_equivalence}
\equivalencelagrangianscpf*
\begin{proof}
If we manipulate the expression of the CPF Lagrangian we can see how minimizing $ \mathcal{L}_{\textnormal{CPF}}(P_{Y|X}, \lambda)$ is equivalent to minimizing $\mathcal{J}_{\textnormal{CPF}}(P_{Y|X}, \gamma)$, where $\gamma = \lambda + 1$. More specifically, %
\begin{align*}
	\arg &\inf_{P_{Y|X} \in \mathcal{P}} \left \lbrace \mathcal{L}_{\textnormal{CPF}}(P_{Y|X}, \lambda) \right \rbrace  \nonumber \\ 
	&= \arg \inf_{P_{Y|X} \in \mathcal{P}} \left \lbrace I(S;Y) - \lambda I(X;Y|S) \right \rbrace \\
	&=  \arg \inf_{P_{Y|X} \in \mathcal{P}} \left \lbrace I(X;Y) - I(X;Y|S) - \lambda I(X;Y|S) \right \rbrace \\
	&=  \arg \inf_{P_{Y|X} \in \mathcal{P}} \left \lbrace I(X;Y) - (\lambda + 1) I(X;Y|S) \right \rbrace \\
	&= \arg \inf_{P_{Y|X} \in \mathcal{P}} \left \lbrace \mathcal{J}_{\textnormal{CPF}}(P_{Y|X}, \lambda + 1) \right \rbrace,
\end{align*}
where $\mathcal{P}$ is the set of probability distributions over $\mathcal{Y}$ such that if $P_{Y|X=x} \in \mathcal{P}$ for all $x \in \mathcal{X}$, then the Markov chain $S \leftrightarrow X \rightarrow Y$ holds.
\end{proof}
\equivalencelagrangianscfb*
\begin{proof}
If we manipulate the expression of the CFB Lagrangian we can see how  minimizing $ \mathcal{L}_{\textnormal{CFB}}(P_{Y|X}, \lambda)$ is equivalent to minimizing $\mathcal{J}_{\textnormal{CFB}}(P_{Y|X}, \beta)$, where $\beta = \lambda + 1$. More specifically, 
\begin{align*}
&\arg \inf_{P_{Y|X} \in \mathcal{P}} \left \lbrace \mathcal{L}_{\textnormal{CFB}}(P_{Y|X}, \lambda) \right \rbrace \nonumber \\
& \ \ \ = \arg \inf_{P_{Y|X} \in \mathcal{P}} \left \lbrace I(S;Y) + I(X;Y|S,T) - \lambda I(T;Y|S) \right \rbrace \\
	&\ \ \ = \arg \inf_{P_{Y|X} \in \mathcal{P}} \left \lbrace I(X;Y) - I(T;Y|S) - \lambda I(T;Y|S) \right \rbrace \\
	&\ \ \ = \arg \inf_{P_{Y|X} \in \mathcal{P}} \left \lbrace I(X;Y) - (\lambda + 1) I(T;Y|S) \right \rbrace \\
	&\ \ \ = \arg \inf_{P_{Y|X} \in \mathcal{P}} \left \lbrace \mathcal{J}_{\textnormal{CFB}}(P_{Y|X}, \lambda + 1) \right \rbrace,
\end{align*}
where $\mathcal{P}$ is the set of probability distributions over $\mathcal{Y}$ such that if $P_{Y|X=x} \in \mathcal{P}$ for all $x \in \mathcal{X}$, then the Markov chains $S \leftrightarrow X \rightarrow Y$ and $T \leftrightarrow X \rightarrow Y$ hold.
\end{proof}

\section{Extended Related Work}
\label{app:related_work}

\subsection{Privacy}
\label{subsec:related_work_privacy}
If the secret information $S$ is the identity of the samples or their membership to a certain group, the field of \emph{differential privacy} (DP) provides a theoretical framework for defining privacy and several mechanisms able to generate privacy-preserving queries about the data $X$ and explore such data, see, e.g., \cite{dwork2014algorithmic}. If, on the other hand, the secret information is arbitrary, variants of DP such as \cite[Bounded DP]{dwork2006calibrating} and \cite[Attribute DP]{kifer2011no} or the theoretical framework introduced in \cite{du2012privacy} are commonly adopted. The privacy funnel~\cite{makhdoumi2014information} is a special case of the latter, when the utility and the privacy are measured with the mutual information.

The original greedy algorithm to compute the PF \cite{makhdoumi2014information} assumes the data is discrete or categorical and do not scale. For this reason,
methods that attempt at learning non-parametric encoding densities $p_{Y|X}$ such as \cite{romanelli2019generating} and other
approaches that take advantage of the scalability of deep learning 
emerged. For instance, \cite{edwards2015censoring} and \cite{hamm2017minimax} learn the representations through adversarial learning but are limited to $S \in \lbrace 0,1 \rbrace$ and do not offer an information theoretic interpretation. Similar to us, in the the privacy preserving variational autoencoder (PPVAE) \cite{nan2020variational} and the unsupervised version of the variational fair autoencoder (VFAE) \cite{louizos2015variational} they learn such representations with variational inference.

At their core, the PPVAE and the unsupervised VFAE end up minimizing the cost functions
\begin{align}
    \mathcal{J}_{\textnormal{PPVAE}}(\theta,\phi,\eta) = \mathbb{E}_{p_{(S,X)|\theta}}[D_{\textnormal{KL}}(p_{Y|(S,X,\theta)}||q_{Y|\theta})] 
    - \eta^{-1} \mathbb{E}_{p_{(S,X,Y)|\theta}}[\log(q_{X|(S,Y,\phi)})]
    \label{eq:eq_ppvae}
\end{align}
and $\mathcal{J}_{\textnormal{VFAE}}(\theta,\phi,\delta) =  \mathcal{J}_{\textnormal{PPVAE}}(\theta,\phi,1) + \delta \mathcal{J}_{\textnormal{MMD}}(\theta,\phi)$, where $\mathcal{J}_{\textnormal{MMD}}(\theta,\phi)$ is a maximum-mean discrepancy term. 
Even though the resulting function to optimize is similar to ours, it is important to note that the encoding density in these works is $p_{Y|(S,X,\theta)}$, which does not respect the problem's Markov chain $S \leftrightarrow X \rightarrow Y$. Therefore, the optimization search space includes representations $Y$ that contain information about the private data $S$ that is not even contained in the original data $X$. Moreover, the private data $S$ is \emph{needed} to generate the representations $Y$, which is problematic since it might not be available during inference.
\subsection{Fairness}
\label{subsec:related_work_fairness}
The field of algorithmic fairness is mainly dominated by the notions of \emph{individual fairness}, where the sensitive data $S$ is the identity of the data samples, and \emph{group fairness}, where $S$ is a binary variable that represents the membership of the data samples to a certain group. There are several approaches that aim at producing classifiers that ensure either of these notions of fairness; 
e.g., discrimination-free naive Bayes \cite{calders2009building}, constrained logistic regression, hinge loss, and support vector machines \cite{zafar2015fairness}, or regularized logistic regression through the Wasserstein distance \cite{jiang2019wasserstein}. 

Other lines of work on algorithmic fairness are based on causal inference \cite{kilbertus2017avoiding, kusner2017counterfactual,nabi2018fair, zhang2018mitigating, wang2019equal, madras2019fairness} and 
data massaging \cite{kamiran2010classification}, where the values of the labels of the training data are changed so that the training data is fair.

The notion of fair representations, introduced by \cite{zemel2013learning}, boosted the advances on algorithmic fairness due to the expressiveness of deep learning. These advances are mainly dominated by adversarial learning \cite{zemel2013learning, edwards2015censoring, zhao2019conditional}, even though there are recent variational approaches, too \cite{louizos2015variational, creager2019flexibly}.

The main differences with the variational approach from \cite{creager2019flexibly} are our simple cost function (which does not require to train an additional adversary discriminator) and that we discard the information that is not necessary to draw inferences about $T$. They generate two representations, $Y_{\textnormal{sens}}$ and $Y_{\textnormal{non-sens}}$, that contain the information about the sensitive data and the original data, respectively, without taking into account the task at hand.
At inference time, the sensitive representations $Y_{\textnormal{sens}}$ are corrupted with noise or discarded, and thus the non-sensitive representations $Y_{\textnormal{non-sens}}$ from \cite{creager2019flexibly} serve a similar purpose to the representations $Y$ obtained with our approach.
Compared to the \emph{variational fair autoencoder} \cite{louizos2015variational}, our encoding density does not require the sensitive information $S$, which might not be available during inference, thus not breaking the Markov chain $S \leftrightarrow X \rightarrow Y$. 

\section{Modification of common algorithms to obtain private and/or fair representations}
\label{app:modification_common_algorithms}
In this section of the appendix, we discuss
the simple changes needed to common representation learning algorithms to implement our proposed variational approach. First, we show how common unsupervised learning algorithms can be modified to the variational approach to 
the CPF, thus generating private representations. Then, we show how common supervised learning algorithms can be modified to the variational approach to the CFB, thus generating fair representations. 
\paragraph{Common unsupervised learning algorithms} The cost function of the $\beta$-VAE \cite{higgins2017beta} and the VIB \cite{alemi2016deep} (when the target variable is the identity of the samples) is
\begin{align}
    \mathcal{F}_{\textnormal{uns}}(\theta,\phi,\eta) = \frac{1}{N} \sum_{i=1}^N D_{\textnormal{KL}} \left( p_{Y|X=x^{(i)},\theta} || q_{Y|\theta} \right) - \eta^{-1}  \mathbb{E}_{p_E} \left[ \log \left( q_{X|Y=f(x^{(i)},E), \phi} (x^{(i)})\right)\right],
    \label{eq:beta_vae_cost}
\end{align}
where $\eta$ is a parameter that controls the trade-off between the compression of the representations $Y$ and their ability to reconstruct the original data $X$. Similarly, the VAE \cite{kingma2013auto} cost function is $\mathcal{F}_{\textnormal{uns}}(\theta,\phi,1)$.
\paragraph{Comparison with the CPF} If we compare~\eqref{eq:beta_vae_cost} with the cost function of the CPF $\tilde{\mathcal{J}}_{\textnormal{CPF}}(\theta,\phi,\gamma)$, we observe that the only difference (provided that $\eta^{-1} = \gamma$) is the decoding density. In the CPF the decoding density of the original data $X$ depends both on the representation $Y$ and on the private attributes $S$, while in~\eqref{eq:beta_vae_cost} it only depends on the representation $Y$. Therefore, the cost function $\mathcal{F}_{\textnormal{uns}}(\theta,\phi,\eta)$ is recovered from the cost function of the CPF in the case that $q_{X|(S,Y,\phi)} = q_{X|(Y,\phi)}$. However, this is not desirable, since it means that the representation $Y$ contains all the information from the private attributes $S$ necessary to reconstruct $X$.
\paragraph{Modifications to obtain private representations} In these usupervised learning algorithms  \cite{kingma2013auto, higgins2017beta, alemi2016deep} the decoding (or generative) density is parametrized with neural networks, e.g., $q_{X|Y,\phi} = \textnormal{Cat}(X;\rho_{\textnormal{dec}}(Y;\phi))$ if $X$ is discrete and $q_{X|Y,\phi} = \mathcal{N}(X;\mu_{\textnormal{dec}}(Y;\phi),\sigma_{\textnormal{dec}}(Y;\phi)^2 I_{d_{\textnormal{dec}}})$ if $X$ is continuous, where $\rho_{\textnormal{dec}}, \mu_{\textnormal{dec}}$, and $\sigma_{\textnormal{dec}}$ are neural networks and $d_{\textnormal{dec}}$ is the dimension of $X$. In this work, the decoding density can also be parametrized with neural networks, e.g., $q_{X|Y,\phi} = \textnormal{Cat}(X;\rho'_{\textnormal{dec}}(S,Y;\phi))$ if $X$ is discrete and $q_{X|Y,\phi} = \mathcal{N}(X;\mu'_{\textnormal{dec}}(S,Y;\phi),\sigma'_{\textnormal{dec}}(S,Y;\phi)^2I_{d_{\textnormal{dec}}})$ if $X$ is continuous, where $\rho'_{\textnormal{dec}}, \mu'_{\textnormal{dec}}$, and $\sigma'_{\textnormal{dec}}$ are neural networks. Therefore, if the decoding density neural networks from \cite{kingma2013auto, higgins2017beta, alemi2016deep} are modified so that they take the private attributes $S$ as an input, then the resulting algorithm is the one proposed in this paper. 
\paragraph{Common supervised learning algorithms} The cost function of the VIB \cite{alemi2016deep} and the nonlinear IB \cite{kolchinsky2019nonlinear} is 
\begin{align}
    \mathcal{F}_{\textnormal{sup}}(\theta,\phi,\eta) = \frac{1}{N} \sum_{i=1}^N D_{\textnormal{KL}} \left( p_{Y|X=x^{(i)},\theta} || q_{Y|\theta} \right)  - \eta^{-1}  \mathbb{E}_{p_E} \left[ \log \left( q_{T|Y=f(x^{(i)},E), \phi} (t^{(i)})\right)\right],
    \label{eq:vib_cost}
\end{align}
where $\eta$ is a parameter that controls the trade-off between the compression of the representations $Y$ and their ability to draw inferences about the task $T$. 
\paragraph{Comparison with the CFB} Similarly to the comparison of the unsupervised learning algorithms and the CPF, we observe that~\eqref{eq:vib_cost} only differs with the cost function $\tilde{\mathcal{J}}_{\textnormal{CFB}}(\theta,\phi, \beta)$ in the decoding density, i.e., the cost function $\mathcal{F}_{\textnormal{sup}}(\theta,\phi,\eta)$ can be recovered from $\tilde{\mathcal{J}}_{\textnormal{CFB}}(\theta,\phi, \beta)$ by setting $q_{T|S,Y,\phi} = q_{T|Y,\phi}$. The inference density of the task $T$ only depends on the representation $Y$ in \cite{alemi2016deep, kolchinsky2019nonlinear}, while in our work it depends both on the representation $Y$ and the sensitive attributes $S$. Hence, in these works the representations
contains all the information from the sensitive attributes $S$ necessary to draw inferences about the task $T$. 
\paragraph{Modifications to obtain fair representations} The argument is analogous to the one for the modifications of unsupervised learning algorithms to obtain private representations. The only modification required in these supervised learning algorithms \cite{alemi2016deep, kolchinsky2019nonlinear} is to modify the decoding density neural networks to receive the sensitive attributes $S$ as an input as well as the representations $Y$.
\paragraph{Invariants of the algorithms} In all these works \cite{kingma2013auto, higgins2017beta, alemi2016deep, kolchinsky2019nonlinear} and ours, the first (or the compression) term is usually calculated assuming that the encoder density is parametrized with neural networks, e.g., $p_{Y|X,\theta} = \mathcal{N}(Y;\mu_{\textnormal{enc}}(X;\theta),\sigma_{\textnormal{enc}}(X;\theta)^2 I_d)$, which allows the representations to be constructed using the reparametrization trick, e.g., $Y = \mu_{\textnormal{enc}}(X;\theta) + \sigma_{\textnormal{enc}}(X;\theta) E$, where $E \sim \mathcal{N}(0,I_d)$, $d$ is the dimension of the representations, and $I_d$ is the $d$-dimensional identity matrix. Then, the marginal density of the representations is set so that the Kullback-Leibler divergence has either a closed expression, a simple way to estimate it, or a simple upper bound, e.g., $q_{Y|\theta} = \mathcal{N}(Y;0, I_d)$ or $q_{Y|\theta} = \frac{1}{N} \sum_{i=1}^N p_{Y|X=x^{(i)},\theta}$, where $x^{(i)}$ are the input data samples. Moreover, the loss function applied to the output of the decoding density and the optimization algorithm, e.g., stochastic gradient descent or Adam \cite{kingma2014adam}, can remain the same in these works and ours, too.
\begin{remark}
The aforementioned modifications can also be introduced in other algorithms with cost functions with additional terms to $\mathcal{F}_{\textnormal{uns}}$ and $\mathcal{F}_{\textnormal{sup}}$. For example, adding a maximum-mean discrepancy (MMD) term on the representation priors to avoid the information preference problem like in the InfoVAE \cite{zhao2019infovae}; adding an MMD term on the encoder densities to enforce privacy or fairness like in the VFAE \cite{louizos2015variational}; or adding a total correlation penalty to the representation's marginal to enforce disentangled representations like in the Factor-VAE, the $\beta$-TCVAE, or the FFVAE \cite{kim2018disentangling, chen2018isolating}, \cite{creager2019flexibly}.
\end{remark}

\section{Details of the experiments}
\label{app:experiments}

In this section of the appendix, we include an additional experiment on the COMPAS dataset \cite{dieterich2016compas} and describe the details of the experiments performed to validate the approach proposed in this paper. The code is at \url{https://github.com/burklight/VariationalPrivacyFairness}.

\subsection{Results on the COMPAS dataset}

\paragraph{COMPAS dataset.} The ProPublica COMPAS dataset \cite{dieterich2016compas}\footnote{Available 
in the Kaggle website.}  contains $6,172$ samples of different attributes of criminal defendants in order to classify if they will recidivate within two years or not. These attributes include \emph{gender}, \emph{age}, or \emph{race}. In both tasks, we followed the experimental set-up from \cite{zhao2019conditional} and considered $S$ to be a binary variable stating if the defendant is African American and $X$ to be the rest of attributes. For the fairness task, we considered $T$ to be the binary variable stating if the defendant recidivated or not. Since this dataset was not previously divided between training and test set, we randomly splitted the dataset with $70$\% of the samples ($4,320$) for training and the rest ($1,852$) for testing.

\begin{table*}[h!]
 \centering
 \caption{{Fairness metrics with different methods on the COMPAS dataset. Displayed as: Logistic regression / Random forest. The best hyperparameters for the other methods have been selected, as shown in the following section. These models are compared to our model with most similar results.}}
 \label{table:comparison_other_methods_compas}
 \begin{tabular}{c|c|c|c|c|c}
    \toprule
    Methods & Accuracy (T) & Accuracy (S) & Discrimination & Error gap & EO gap \\
    \midrule
    LFR \cite{zemel2013learning} & 0.64 / 0.67 & 0.51 / \textbf{0.99} & 0.14 / \textbf{0.21} & 0.07 / 0.03 & 0.11 / \textbf{0.19} \\
    Ours ($\beta = 34.8$) & 0.63 / 0.58 & 0.59 / 0.52 & 0.16 / 0.10 & 0.06 / 0.04 & 0.13 / 0.08\\
    \midrule 
    FFVAE \cite[$\alpha=200$]{creager2019flexibly} & 0.54 / 0.51 & 0.50 / 0.50 & 0.00 / 0.01 & 0.12 / 0.05 & 0.00 / 0.04 \\
    CFAIR \cite[$\lambda=10$]{zhao2019conditional}\footnotemark  & - & - & 0.04 & 0.00 & 0.02 \\
    Ours ($\beta = 7.76$) & 0.54 / 0.52 & 0.52 / 0.49  & 0.00 / 0.02  & 0.14 / 0.04  & 0.00 / 0.03 \\
    \bottomrule
 \end{tabular}
\end{table*}
\footnotetext{Results extracted from the original paper's Figure 2. The fairness metrics were computed on the output of the discriminator network, not on an independent logistic regression or random forest.}

Similarly to the previous experiments, the proposed approach controls the trade-off between private and informative representations and between fair and accurate representations. In Figure \ref{fig:compas_trade_off_information} we see how the trade-off between the compression level $I(X;Y)$ and the informativeness of the representations independent of the private data $I(X;Y|S)$ and between the compression level $I(X;Y)$ and the predictability of the representations without  the sensitive data $I(X;Y|S)$ is controlled by the private and the fair representations, respectively. Moreover, we can also see how the amount of information the representations keep about the private or the sensitive data is commanded by the Lagrange multipliers $\gamma$ and $\beta$.

Compared with other variational approaches to the PF \cite[PPVAE]{nan2020variational} and \cite[VFAE]{louizos2015variational}, as happened with the Adult dataset, the proposed approach controlled better the information the representations $Y$ contained about the sensitive attribute $S$. In particular, the PPVAE contained between 0.63 and 1 bit of information about $S$ for $\eta^{-1} \in [1,50]$, and the VFAE contained between 0.68 and 1 bit for $\delta \in [N_{\textnormal{batch}},1000 N_{\textnormal{batch}}]$. Note that 1 bit is the maximum information that $Y$ can contain about $S$, since $H(S) = 1$ in this scenario.  With respect to membership attacks, our method was a slightly weaker to linear attackers than the PPVAE and the VFAE, allowing an accuracy in the range $[0.54,0.60]$, compared to their respective ranges of $[0.46, 0.62]$ and $[0.45,0.59]$. That is, considering the prior probability of $0.51$, the best parameter of our method conceded the attacker a $3\%$ accuracy that the other methods did not allow. However, once the attacker employed more sophisticated attacks, such as a random forest, our method maintained a range of $[0.52,0.56]$, while the PPVAE and the VFAE allowed almost a perfect recovery of the group membership, with respective accuracy ranges between $0.96$ and $0.99$ and between $0.96$ and $1.0$, as was suggested by the amount of bits of information their representations contained about $S$. As before, this behavior can be explained due to the Markov chain violation of the encoder densities $p_{Y|S,X,\theta}$ of these approaches.

Furthermore, the Lagrange multiplier $\beta$ also allows us to control the behavior of the accuracy, the error gap, and the discrimination for the COMPAS dataset (Figures \ref{fig:compas_betas_vs_accuracy}, \ref{fig:compas_betas_vs_err_gap}, and \ref{fig:compas_betas_vs_discrimination}). Moreover, in this scenario, as shown in Figures \ref{fig:compas_betas_vs_eq_odds} and \ref{fig:compas_betas_vs_accuracy_s}, an increase of $\beta$ also increased the equalized odds level and the accuracy on $S$ of adversarial classifiers (even though they remained below their values obtained with the original data $X$ for all the $\beta$ tested). These results on the equalized odds, even though not generalizable since we have the counter-example of the Adult dataset, indicate that in some situations this quanitty can be controlled with our approach. More specifically, we believe this happens when we can guarantee that $I(S;Y;T)$ is non-negative as explained in Remark \ref{remark:cfb_and_cpf_equalized_odds}.

Finally, we also observe, as for the Adult dataset, how the proposed method's representations are more robust against non-linear adversaries than the representations obtained with the baseline from \cite[LFR]{zemel2013learning} (see Table~\ref{table:comparison_other_methods_compas}). Moreover, the method performs similarly as state-of-the-art methods based on adversarial learning \cite[CFAIR]{zhao2019conditional} and variational inference \cite[FFVAE]{creager2019flexibly}, as shown in Table~\ref{table:comparison_other_methods_compas}.

\begin{figure}[ht]
    \centering
    \subfloat[\label{fig:compas_ixy_vs_minus_hx_given_sy}]{{\includegraphics[width=0.33\textwidth]{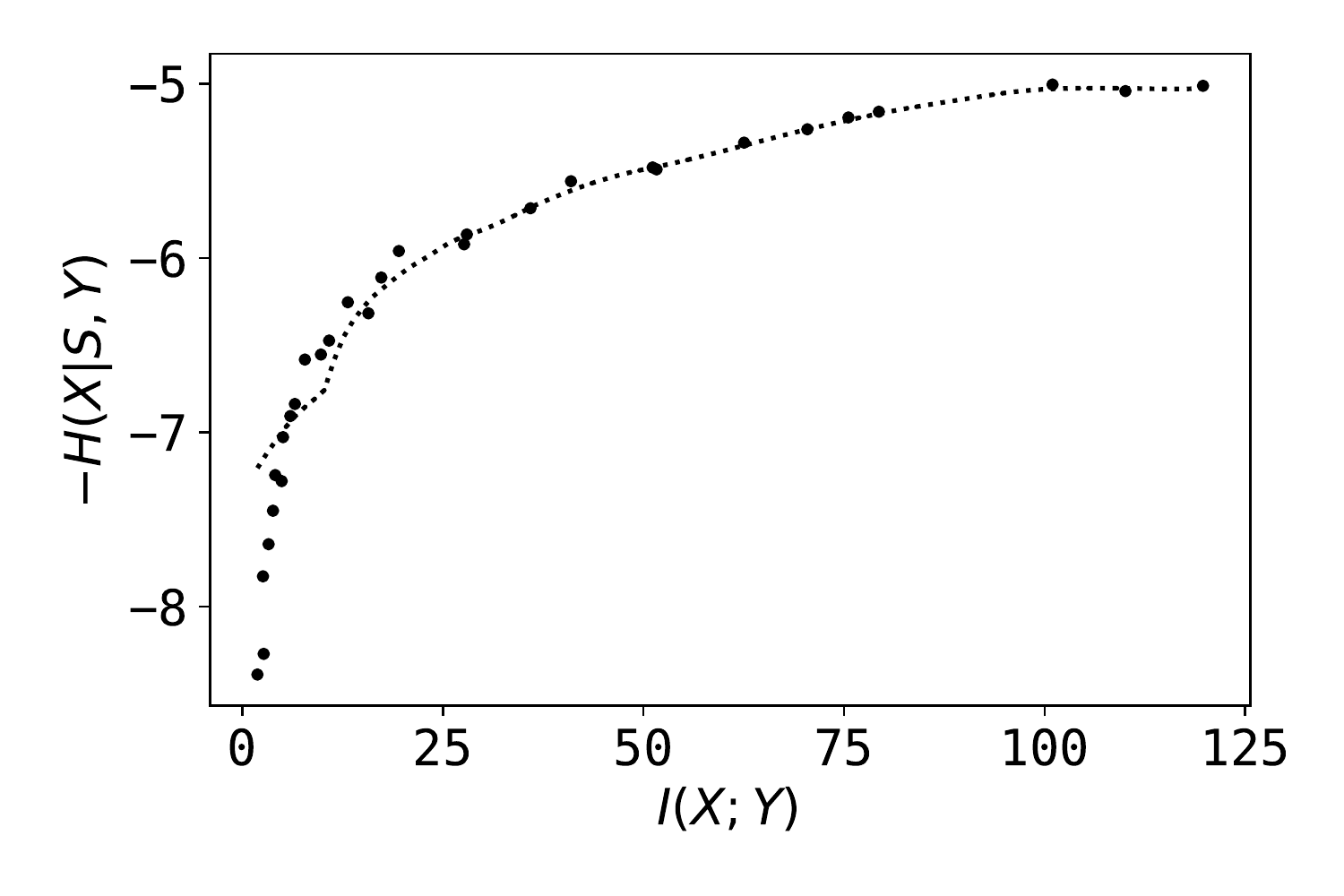}}}  
    \subfloat[\label{fig:compas_gammas_vs_isy}]{{\includegraphics[width=0.33\textwidth]{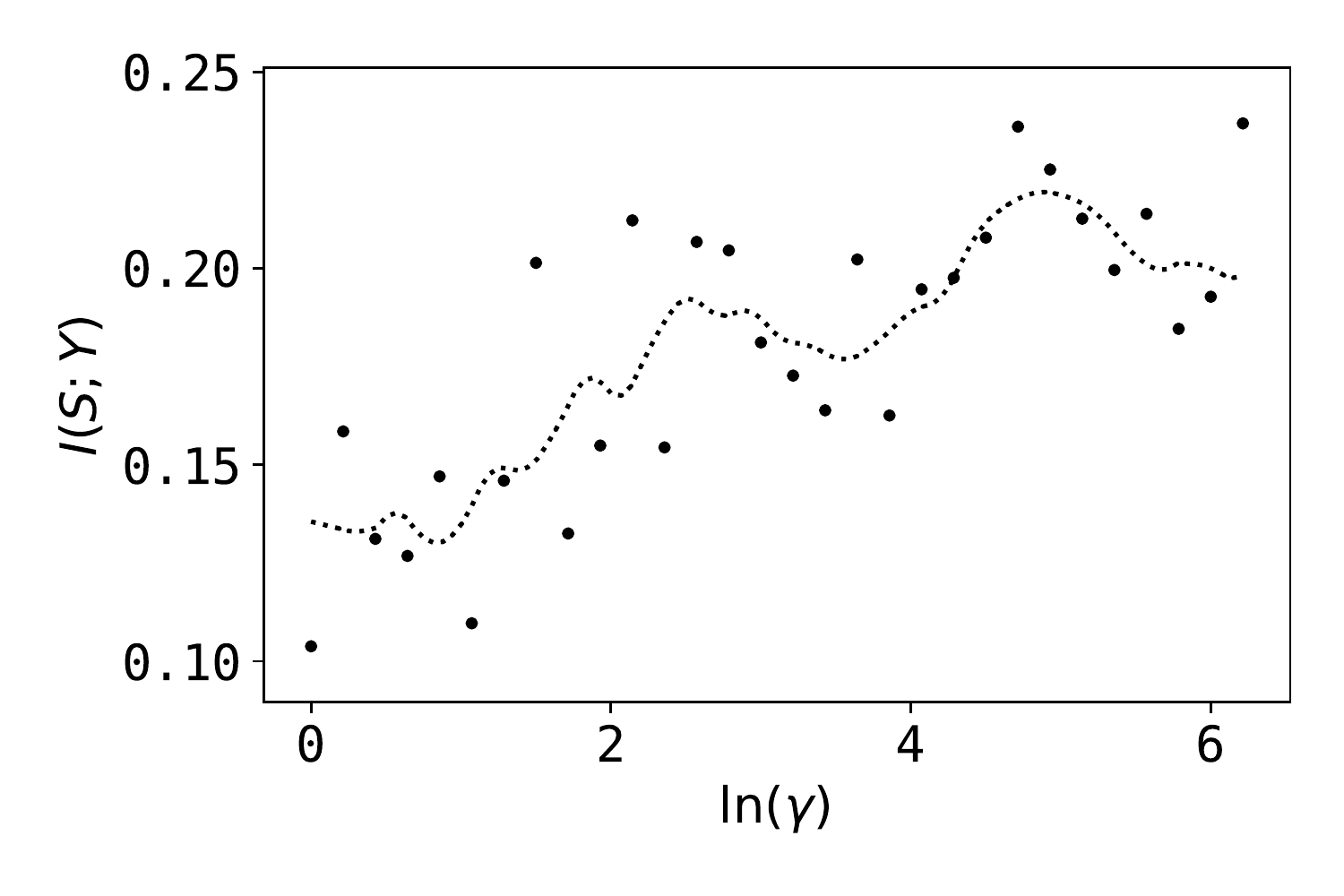}}}
    
    \subfloat[\label{fig:compas_betas_vs_isy}]{{\includegraphics[width=0.33\textwidth]{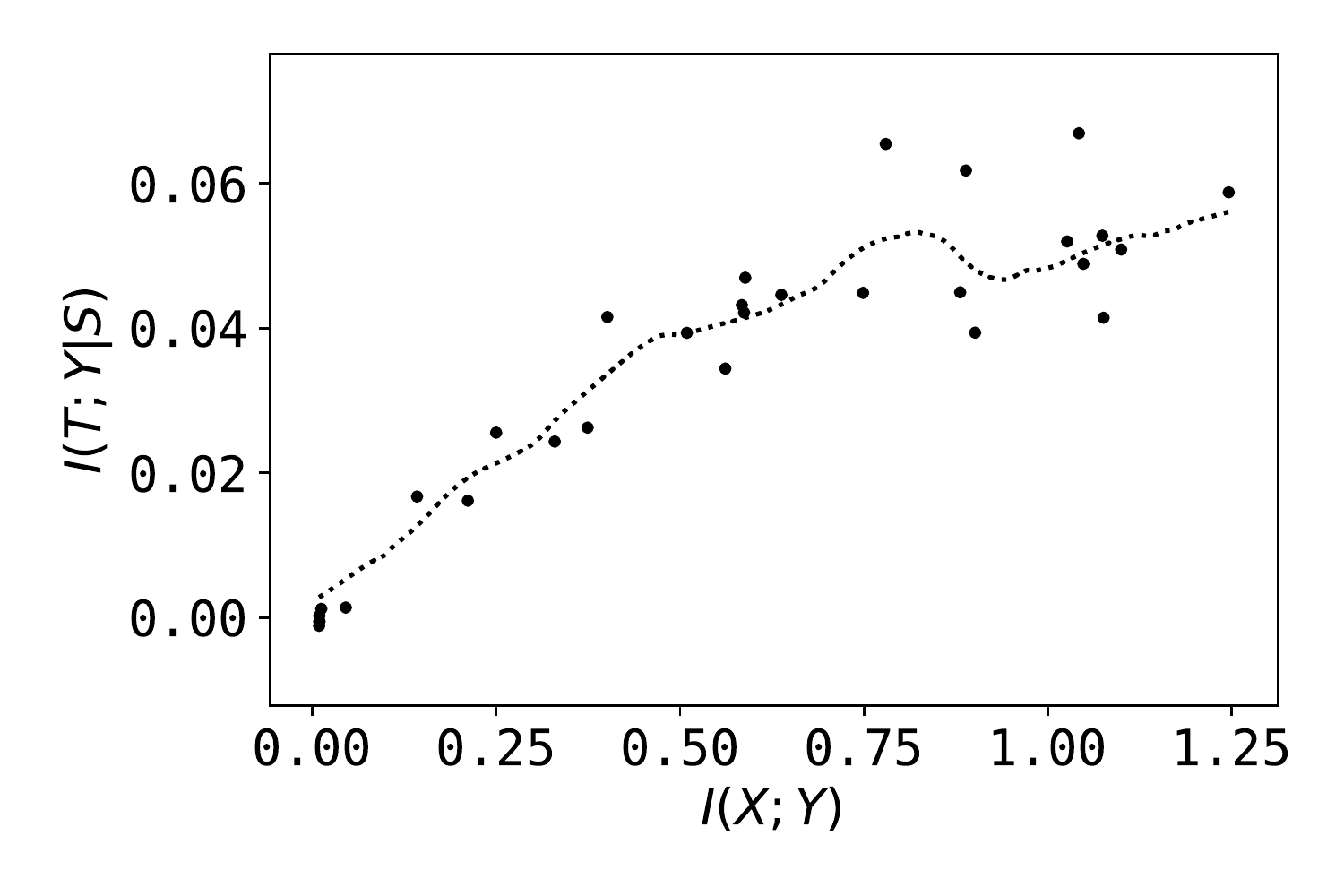}}}
    \subfloat[\label{fig:compas_mnist_betas_vs_isy}]{{\includegraphics[width=0.33\textwidth]{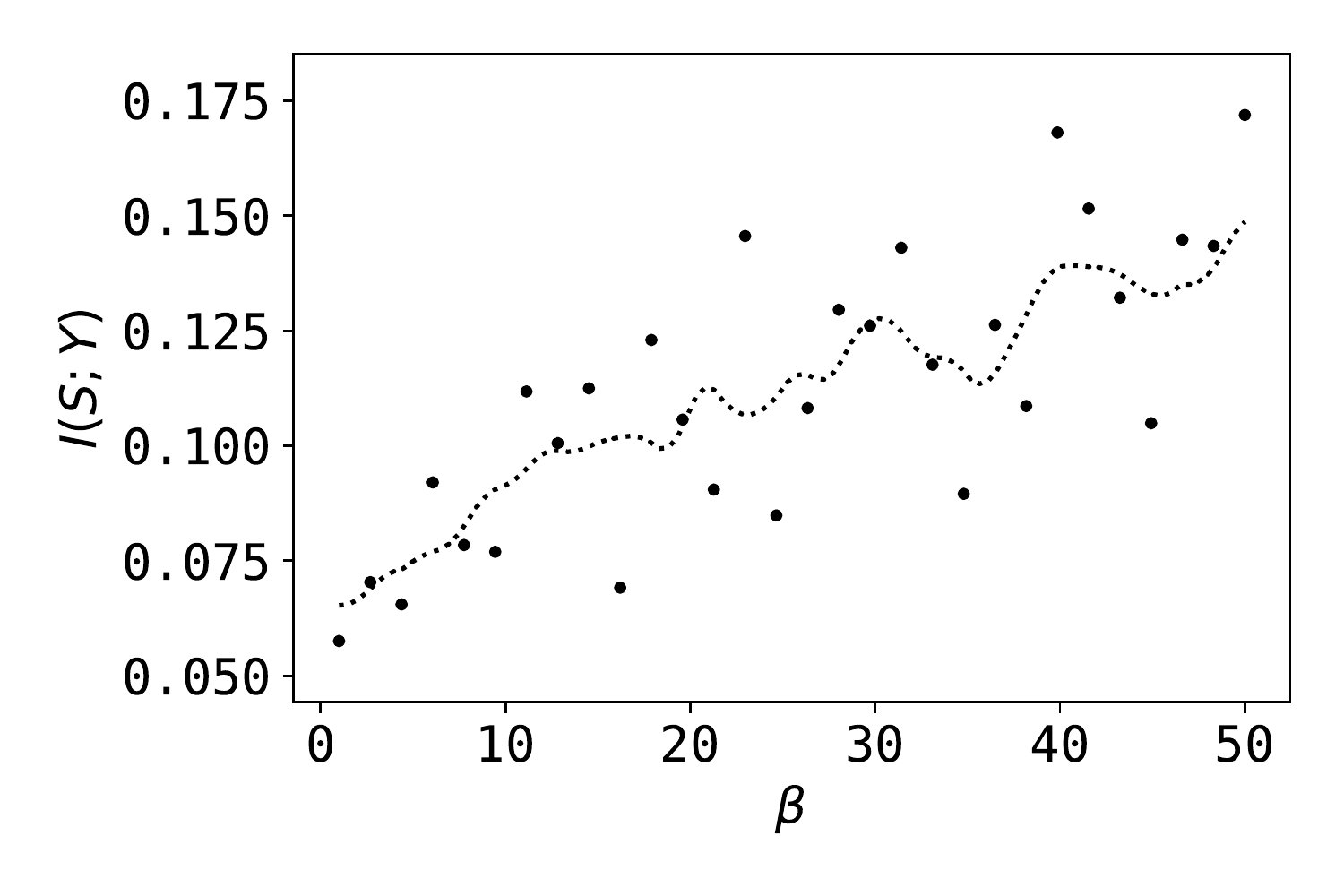}}}
    \caption{Trade-off between (a) the private representations compression $I(X;Y)$ and the non-private information retained $I(X;Y|S)$ and (b) the Lagrange multiplier $\gamma \in [1,500]$ and the private information $I(S;Y)$ kept by the representations. Since $I(X;Y|S) = H(X|S) - H(X|S,Y)$ and $H(X|S)$ does not depend on $Y$, the reported quantity is $-H(X|S,Y)$. Moreover, trade-off between (c) the fair representations compression $I(X;Y)$ and the non-sensitive information retained about the task $I(T;Y|S)$ and (d) the Lagrange multiplier $\beta \in [1,50]$ and the sensitive information kept by the representations. All quantities are obtained for the COMPAS dataset. The dots are the computed empirical values and the lines are the moving average of the 1D linear interpolations of such points.}
    \label{fig:compas_trade_off_information}
\end{figure}

\begin{figure*}[ht]
	\centering
    \subfloat[\label{fig:compas_betas_vs_accuracy}]{{\includegraphics[width=0.32\textwidth]{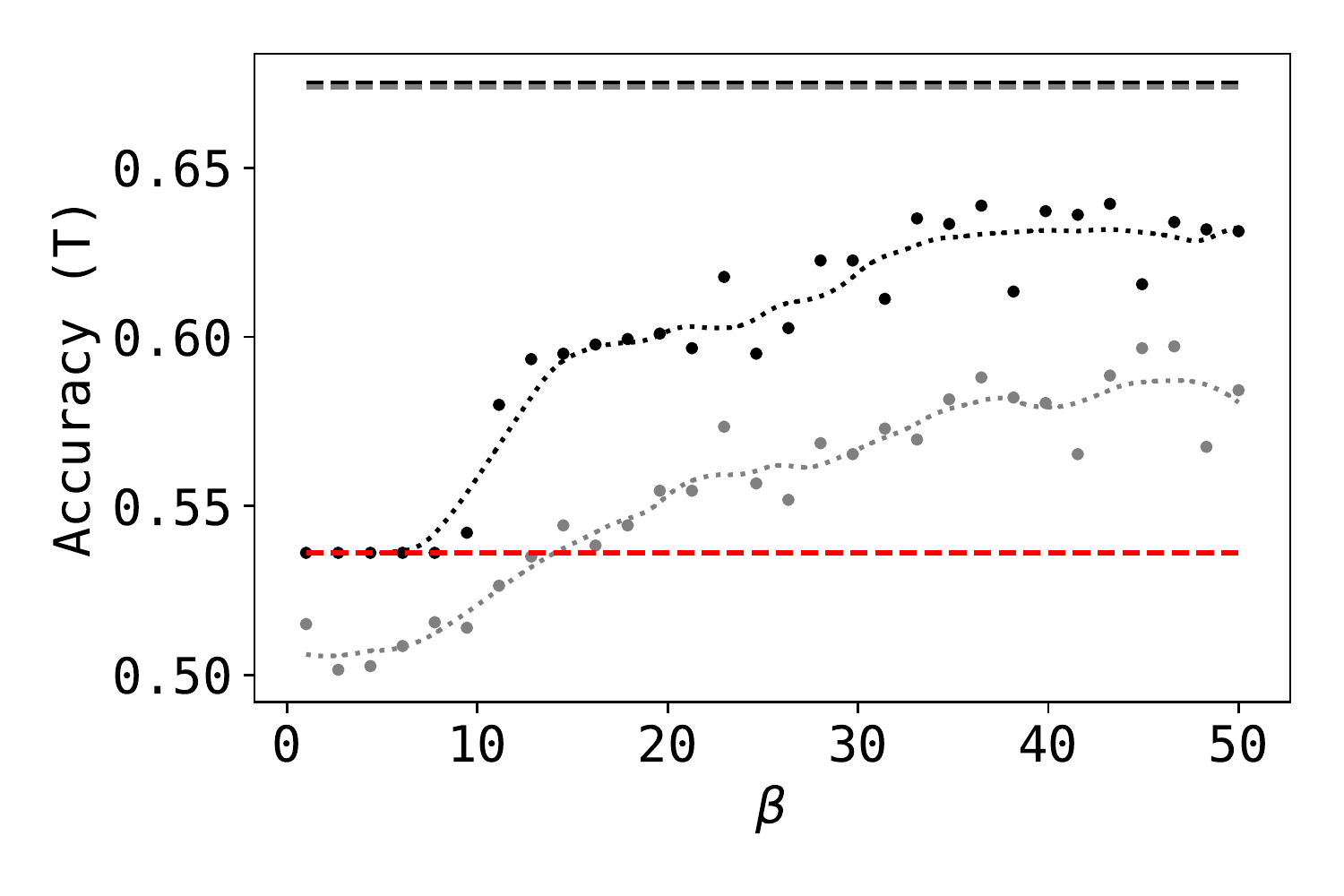}}}
    \subfloat[\label{fig:compas_betas_vs_discrimination}]{{\includegraphics[width=0.32\textwidth]{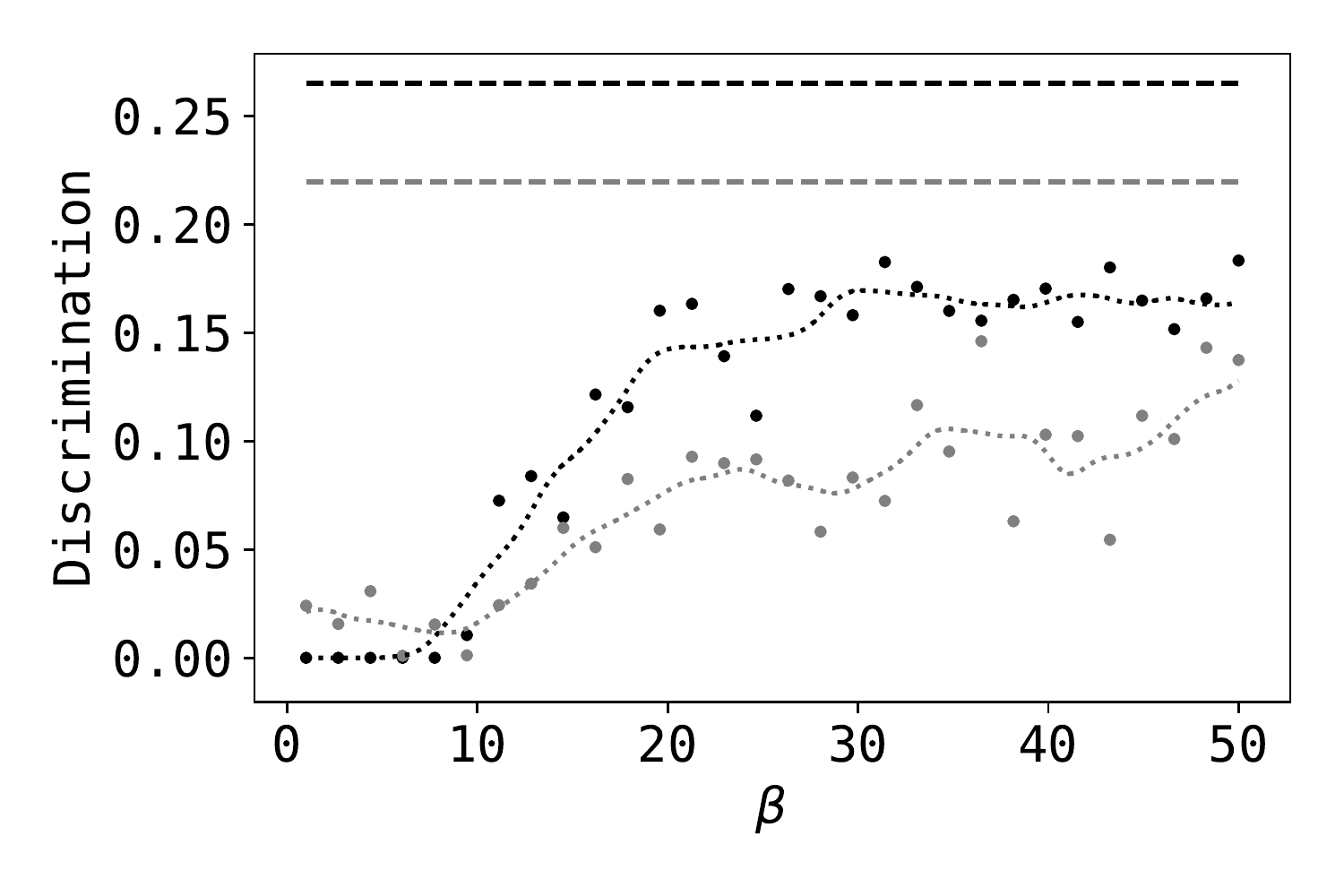}}}
    \subfloat[\label{fig:compas_betas_vs_eq_odds}]{{\includegraphics[width=0.32\textwidth]{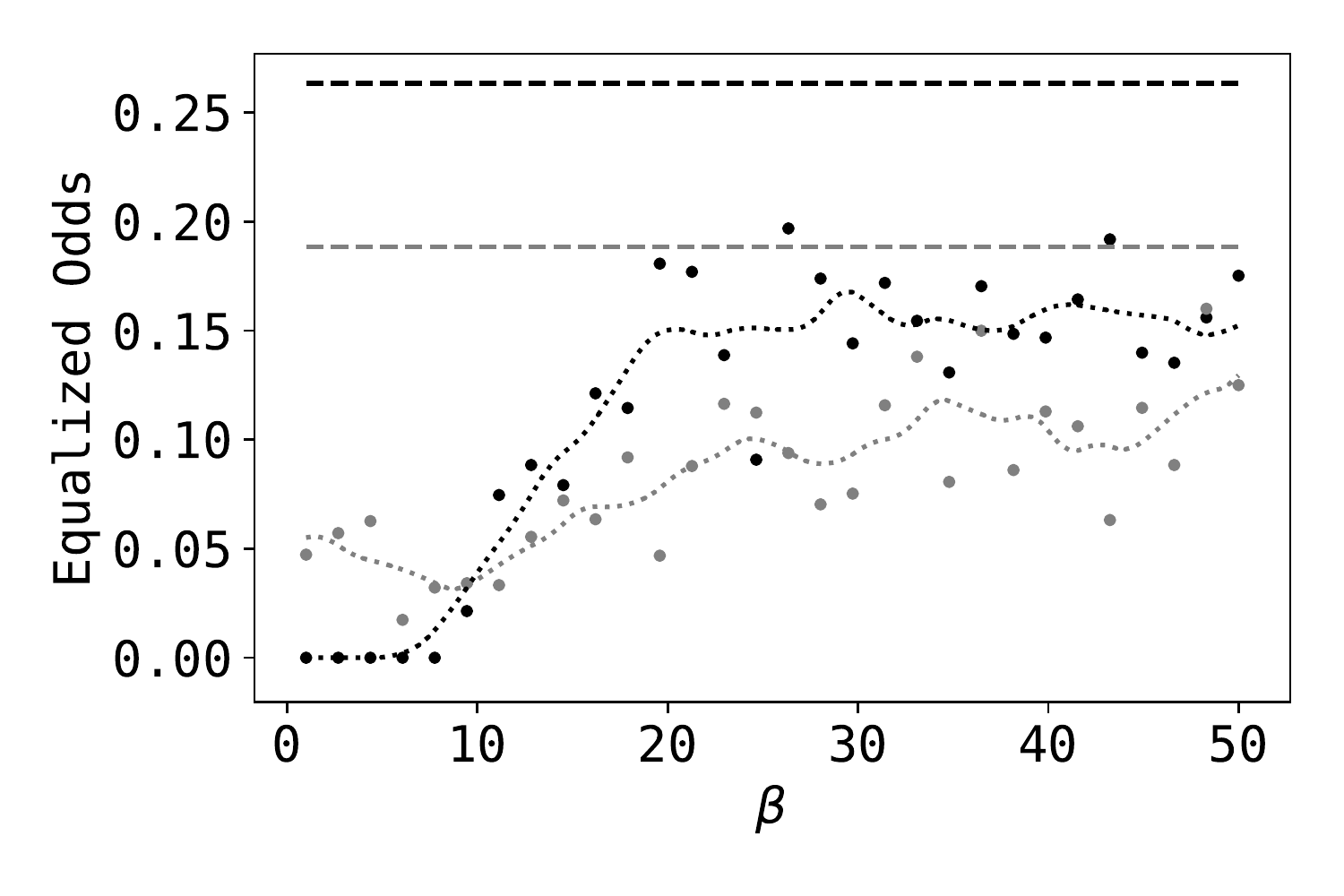}}}
    
    \subfloat[\label{fig:compas_betas_vs_err_gap}]{{\includegraphics[width=0.32\textwidth]{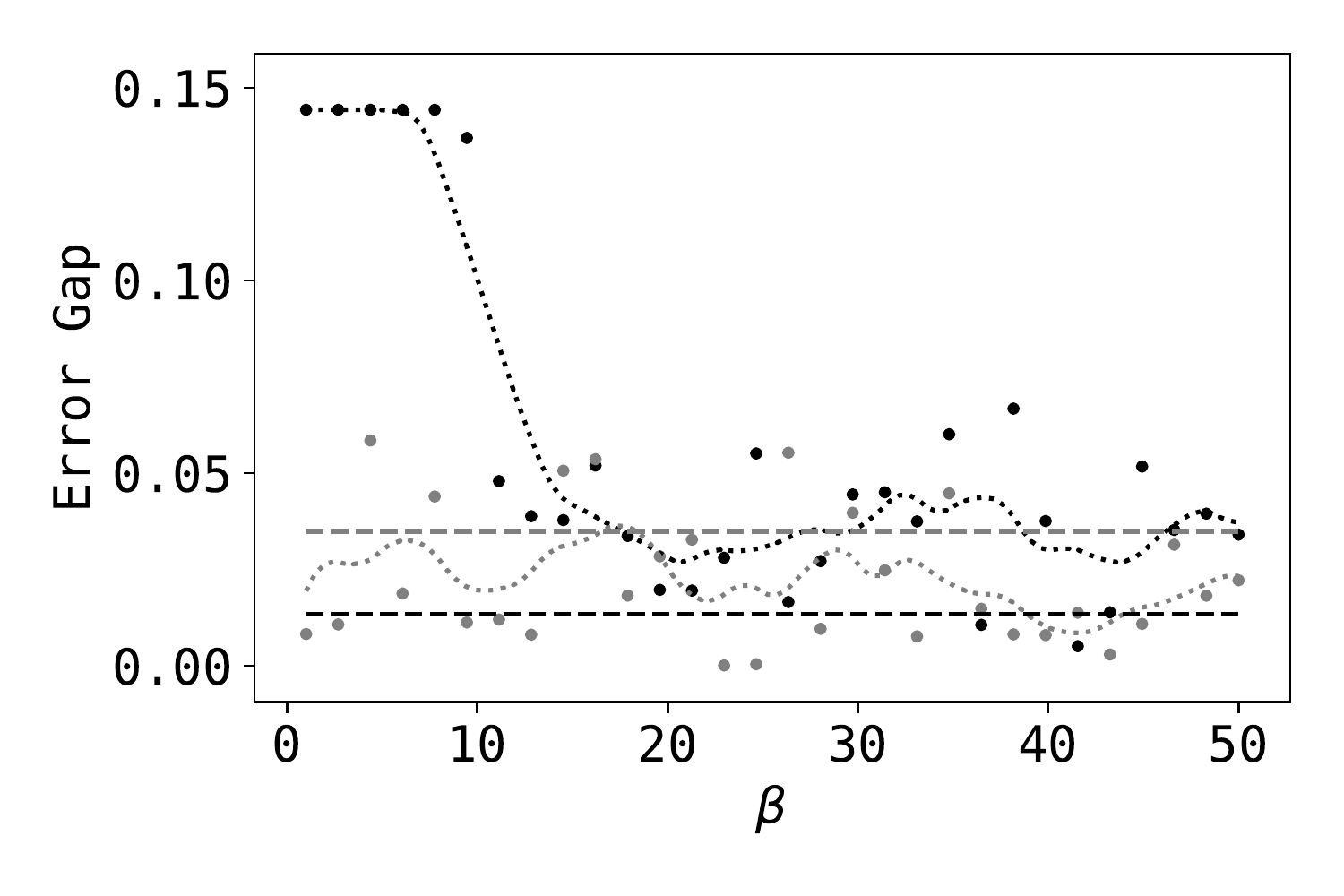}}}
    \subfloat[\label{fig:compas_betas_vs_accuracy_s}]{{\includegraphics[width=0.32\textwidth]{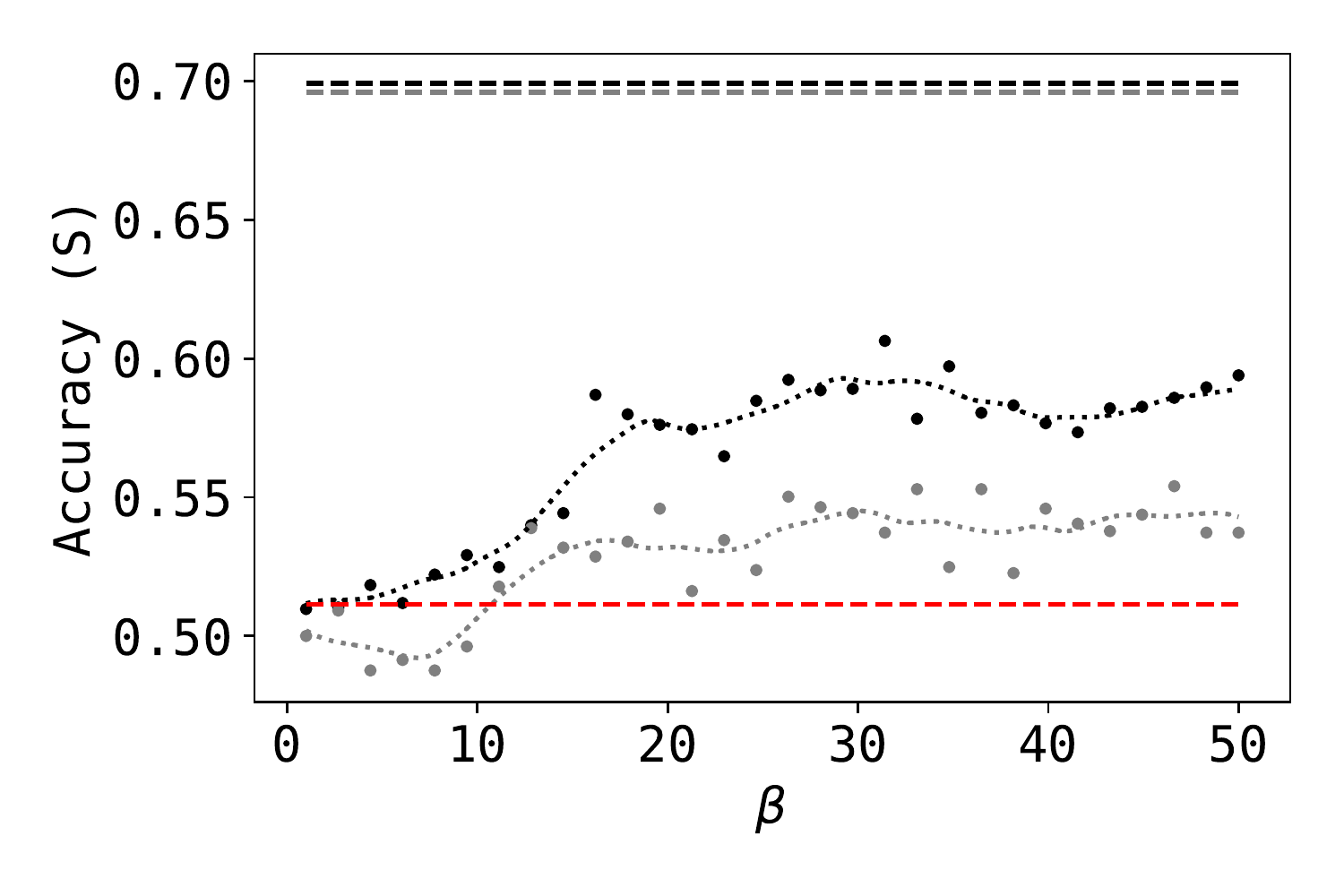}}}
    \subfloat{{\includegraphics[width=0.17\textwidth]{Figures/adult/legend.pdf}}}
    \caption{
    Behavior of (a) the accuracy on $T$, (b) the discrimination, (c) the equalized odds gap, (d) the error gap, and (e) the accuracy on $S$ of a logistic regression (LR, in black) and a random forest (RF, in gray) of the fair representations (dots and dotted line) and the original data (dashed line) learned with $\beta \in [1,50]$ on the COMPAS dataset. The dashed line in red is the accuracy of a prior-based classifier.}
\end{figure*}

\subsection{Experimental details}

\paragraph{Encoders} In all the experiments performed, we modeled the encoding density as an isotropic Gaussian distribution, i.e., $p_{Y|X,\theta} = \mathcal{N}(Y;\mu_{\textnormal{enc}}(X;\theta),\sigma_{\theta}^2 I_2)$, so that $Y = \mu_{\textnormal{enc}}(X;\theta) + \sigma_\theta E$, where $E \sim \mathcal{N}(0, I_2)$, $\mu_{\textnormal{enc}}$ is a neural network, $\sigma_\theta$ is also optimized via gradient descent but is not calculated with $X$ as an input, and where the representations have $2$ dimensions. The neural networks in each experiment were: 
\begin{itemize}
    \item For the Adult dataset, $\mu_{\textnormal{enc}}$ was a multi-layer perceptron with a single hidden layer with 100 units and ReLU activations.
    \item For the Colored MNIST dataset, $\mu_{\textnormal{enc}}$ was the convolutional neural network \texttt{CNN-enc-1} for both the privacy and fairness experiments, and the convolutional neural network \texttt{CNN-enc-2} for the example from Figure \ref{fig:example_cpf}. Both architectures are described in Table \ref{table:cnn_architectures}.
    \item For the COMPAS dataset, $\mu_{\textnormal{enc}}$ was a multi-layer perceptron with a single hidden layer with 100 units and ReLU activations.
\end{itemize}
Moreover, the marginal density of the representations was modeled as an isotropic Gaussian of unit variance and zero mean; i.e., $q_{Y|\theta} = \mathcal{N}(Y;0,I_2)$.

\begin{table*}[ht]
  \caption{Convolutional neural network architectures employed for the Colored MNIST dataset. The network modules are the following: \texttt{Conv2D(cin,cout,ksize,stride,pin,pout)} and \texttt{ConvTrans2D(cin,cout,ksize,stride,pin,pout)} represent, respectively, a 2D convolution and transposed convolution, where \texttt{cin} is the number of input channels, \texttt{cout} is the number of output channels, \texttt{ksize} is the size of the filters, \texttt{stride} is the stride of the convolution, \texttt{pin} is the input padding of the convolution, and \texttt{pout} is the ouptut padding of the convolution; \texttt{MaxPool2D(ksize,stride,pin)} represents a max-pooling layer, where the variables mean the same than for the convolutional layers; \texttt{Linear(nu)} represents a linear layer, where \texttt{nu} are the output units; and \texttt{BatchNorm}, \texttt{ReLU6}, \texttt{Flatten}, \texttt{Unflatten}, and \texttt{Sigmoid} represent a batch normalization, ReLU6, flatten, unflatten and Sigmoid layers, respectively.}
\label{table:cnn_architectures}
  \centering
  \begin{tabular}{l|p{145mm}}
    \toprule
    Name & Architecture \\
    \midrule
    \texttt{CNN-enc-1} & \texttt{Conv2D(3,5,5,2,1,0)~- BatchNorm~- RelU6~- Conv2D(5,50,5,2,0,0)~- BatchNorm~- RelU6~- Flatten~- Linear(100)~- BatchNorm~- ReLU~- Linear(2)} \\
     \texttt{CNN-enc-2} & \texttt{Conv2D(3,5,5,0,2,0)~- BatchNorm~- RelU6~- Conv2D(3,5,5,0,2,0)~- BatchNorm~- RelU6~- Conv2D(3,5,5,0,2,0)}  \\ 
     \texttt{CNN-dec-1} & \texttt{Linear(100)~- BatchNorm~- ReLU6~- Linear(1250)~- Unflatten~- BatchNorm~- ReLU~- ConvTrans2D(50,5,5,2,0,0)~- BatchNorm~- ReLU~- ConvTrans2D(5,3,5,2,1,1)~- Sigmoid} \\
     \texttt{CNN-dec-2} & \texttt{Conv2D(3,5,5,0,2)~- BatchNorm~- RelU6~- Conv2D(5,50,5,0,2,0)~- BatchNorm~- RelU6~- Conv2D(5,50,5,0,2,0)~- BatchNorm~- RelU6~- Conv2D(5,50,5,0,2,0)~- Sigmoid} \\
     \texttt{CNN-mine} & \texttt{Conv2D(3,5,5,1,1,0)~- MaxPool2D(5,2,2)~- RelU6~- Conv2D(5,50,5,1,0,0)~- MaxPool2D(5,2,2)~- RelU6~- Flatten~- Linear(100)~- ReLU6~- Linear(50)~- ReLU6~- Linear(1)}\\
    \bottomrule
  \end{tabular}
\end{table*}

\paragraph{Decoders} In all the experiments performed for the fairness problem, the target task variable $T$ was binary. Hence, we modeled the inference density with a Bernoulli distribution\footnote{Note that the Bernoulli distribution is a categorical distribution with two possible outcomes.}; i.e., $q_{T|S,X,\phi} = \textnormal{Bernoulli}(T;\rho_{\textnormal{dec}}(S,Y;\phi))$, where $\rho_{dec}$ is a neural network with a Sigmoid activation function in the output. In the privacy problem, if $X$ was a collection of random variables $(X_1, X_2, ..., X_C)$, the generative density was modeled as the product of $C$ categorical and isotropic Gaussians, depending if the variables were discrete or continuous. That is, $q_{X|(S,Y,\phi)} = \prod_{j=1}^C \textnormal{Cat}(X_j;\rho_{\textnormal{dec},j}(S,Y;\phi))
^{\mathbb{I}[X_j = \textnormal{Discrete}]}$ $\mathcal{N}(X_j;\mu_{\textnormal{dec},j}(S,Y;\phi), \sigma_{\phi,j}^2)^{\mathbb{I}[X_j = \textnormal{Continuous}]}$, where the continuous random variables are $1$-dimensional. In practice, the densities were parametrized with a neural network with $K = \sum_{j=1}^C K_j^{\mathbb{I}[X_j = \textnormal{Discrete}]} 1^{\mathbb{I}[X_j = \textnormal{Continuous}]}$ output neurons, where $K_j$ are the number of classes of the categorical variable $X_j$ and each group of output neurons defines each multiplying density; either as the logits of the $K_j$ classes or the parameter (mean) of the Gaussian (the variances were also optimized via gradient descent but were not calculated with $S$ nor $Y$ as an input). If $X$ was an image, the generative density was modeled as a product of $3C$ Bernoulli densities, where $C$ is the number of pixels of the image and the $3$ comes from the RGB channels. The neural networks in each experiments were:
\begin{itemize}
    \item For the Adult dataset, the decoding neural network was a multi-layer perceptron with a single hidden layer with 100 units and ReLU activations. For the fairness task, the output was 1-dimensional with a Sigmoid activation function. For the privacy task, the output was $121$-dimensional. The input of the network was a concatenation of $Y$ and $S$.
    \item For the Colored MNIST dataset and the fairness task, the decoding neural network was also a multi-layer perceptron with a single hidden layer with 100 units, ReLU activations, and a 1-dimensional output with a Sigmoid activation function. For the privacy task, the decoding neural network was the \texttt{CNN-dec-1} for the normal experiments and the \texttt{CNN-dec-2} for the example of Figure \ref{fig:example_cpf}. The input linear layers took as an input a concatenation of $Y$ and $S$ and in the convolutional layers $S$ was introduced as a bias.
    \item For the COMPAS dataset, the decoding neural network also was a multi-layer perceptron with a single hidden layer with 100 units and ReLU activations. For the fairness task, the output was 1-dimensional with a Sigmoid activation function. For the privacy task, the output was $19$-dimensional. The input of the network was a concatenation of $Y$ and $S$.
\end{itemize}
\paragraph{Hyperparameters and random seed} The hyperparameters employed in the experiments to train the encoder and decoder networks are displayed in Table \ref{table:hyperparameters}, and the optimization algorithm used was Adam \cite{kingma2014adam}. All random seeds were set to 2020. 

\paragraph{Hardware and software} All experiments were run with PyTorch \cite{paszke2019pytorch}, NumPy \cite{harris2020array}, and scikit-learn \cite{scikit-learn} on a Nvidia Tesla P100 PCIE GPU of 16Gb of RAM.

\begin{table*}[h!]
  \caption{Hyperparameters employed to optimize the encoder and decoder networks of the experiments.}
\label{table:hyperparameters}
  \centering
  \begin{tabular}{l|c c c c}
    \toprule
    Dataset (task) & Epochs & Learning rate & Lagrange multiplier ($\beta$ or $\gamma$) & Batch size \\
    \midrule
    Adult (fairness) & 150 & $10^{-3}$ & 1-50 (logarithmically spaced) & 1024 \\
    Adult (privacy) & 150 & $10^{-3}$ & 1-50 (logarithmically spaced) & 1024 \\
    Colored MNIST (fairness) & 250 & $10^{-3}$ & 1-50 (logarithmically spaced) & 1024 \\
    Colored MNIST (privacy) & 500 & $10^{-3}$ & 1-50 (logarithmically spaced) & 1024 \\
    Colored MNIST (example) & 500 & $10^{-3}$ & 1 & 2048 \\
    COMPAS (fairness) & 150 & $10^{-4}$ & 1-50 (linearly spaced) & 64 \\
    COMPAS (privacy) & 250 & $10^{-4}$ & 1-500 (logarithmically spaced) & 64\\
    \bottomrule
  \end{tabular}
\end{table*}

\begin{table*}[h!]
  \caption{Hyperparameters employed to optimize the MINE networks of the experiments.}
\label{table:hyperparameters_mine}
  \centering
  \begin{tabular}{l|c c c}
    \toprule
    Dataset (task) & Iterations & Learning rate & Batch size\\
    \midrule
    Adult (fairness) & $5\cdot10^4$ & $10^{-3}$ & 2048 \\
    Adult (privacy) & $5\cdot10^4$ & $10^{-3}$ & 2048 \\
    Colored MNIST (fairness) & $5\cdot10^4$ & $10^{-3}$ & 2048 \\
    Colored MNIST (privacy) & $5\cdot10^4$ & $10^{-3}$ & 2048 \\
    COMPAS (fairness) & $5\cdot10^4$ & $10^{-3}$ & 463 \\
    COMPAS (privacy) & $5\cdot10^4$ & $10^{-3}$ & 463 \\
    \bottomrule
  \end{tabular}
\end{table*}

\begin{table*}[ht!]
 \centering
 \caption{{Fairness metrics with different values of $\alpha$ for the FFVAE. Displayed as: Logistic regression / Random forest.}}
 \label{table:ffvae_alpha_performance}
 \begin{tabular}{c|c|c|c|c|c|c}
    \toprule
    Dataset & $\alpha$ & Accuracy (T) & Accuracy (S) & Discrimination & Error gap & EO gap \\
    \midrule
    Adult & 1 & 0.76 / 0.74 & 0.66 / 0.64 & 0.07 / 0.10 & 0.18 / 0.16 & 0.10 / 0.12 \\ 
    Adult & 100 & 0.77 / 0.75 & 0.67 / 0.61 & 0.00 / 0.01 & 0.17 / 0.11 & 0.07 / 0.02 \\ 
    Adult & 200 & 0.77 / 0.75 & 0.67 / 0.61 & 0.00 / 0.01 & 0.17 / 0.12 & 0.04 / 0.04 \\ 
    Adult & 300 & 0.78 / 0.75 & 0.67 / 0.62 & 0.01 / 0.01 & 0.16 / 0.12 & 0.06 / 0.03 \\ 
    Adult & 400 & 0.78 / 0.75 & 0.67 / 0.61  & 0.01 / 0.01 & 0.16 / 0.12 & 0.05 / 0.05 \\
    \midrule
    COMPAS & 1 & 0.54 / 0.50 & 0.50 / 0.50 & 0.00 / 0.03 & 0.12 / 0.04 & 0.00 / 0.05 \\ 
    COMPAS & 100 & 0.54 / 0.49 & 0.52 / 0.50 & 0.00 / 0.02 & 0.12 / 0.05 & 0.00 / 0.05 \\ 
    COMPAS & 200 & 0.54 / 0.51 & 0.50 / 0.50 & 0.00 / 0.01 & 0.12 / 0.05 & 0.00 / 0.04 \\ 
    COMPAS & 300 & 0.54 / 0.52 & 0.52 / 0.49 & 0.00 / 0.05 & 0.12 / 0.01 & 0.00 / 0.06 \\ 
    COMPAS & 400 & 0.54 / 0.50 & 0.52 / 0.50  & 0.00 / 0.03 & 0.12 / 0.05 & 0.00 / 0.06 \\
    \bottomrule
 \end{tabular}
\end{table*}

\begin{table*}[ht!]
 \centering
 \caption{{Fairness metrics with different values of $\lambda$ for the CFAIR. Displayed as: Logistic regression / Random forest.}}
 \label{table:cfair_lambda_performance}
 \begin{tabular}{c|c|c|c|c|c|c}
    \toprule
    Dataset & $\lambda$ & Accuracy (T) & Accuracy (S) & Discrimination & Error gap & EO gap \\
    \midrule
    Adult & 0.1 & 0.78 / 0.74 & 0.68 / 0.67 & 0.03 / 0.03 & 0.14 / 0.12  & 0.01 / 0.00 \\
    Adult & 1 & 0.77 / 0.71 & 0.68 / 0.63 & 0.03 / 0.04 & 0.17 / 0.09 & 0.17 / 0.11 \\ 
    Adult & 10 & 0.77 / 0.70 & 0.67 / 0.64 & 0.00 / 0.03 & 0.17 / 0.10 & 0.04 / 0.04 \\ 
    Adult & 100 & 0.77 / 0.71 & 0.67 / 0.63 & 0.05 / 0.03 & 0.17 / 0.09 & 0.22 / 0.08 \\ 
    Adult & 1000 & 0.77 / 0.71 & 0.69 / 0.65 & 0.02 / 0.01 & 0.17 / 0.12 & 0.02 / 0.04 \\
    \bottomrule
 \end{tabular}
\end{table*}

\paragraph{Preprocessing} The input data $X$ for the Adult and the COMPAS dataset was normalized to have 0 mean and unit variance. The input data for the Colored MNIST dataset was scaled to the range $[0,1]$.
\paragraph{Information measures} The mutual information $I(X;Y)$ and the conditional entropy $I(X;Y|S)$ and $I(T;Y|S)$ were calculated with the bounds from~\eqref{eq:bound_ixy}, \eqref{eq:bound_ixy_s}, and~\eqref{eq:bound_iyt_s}, respectively. Since $H(X|S)$ was not directly obtainable, $-H(X|S,Y)$ was calculated and displayed instead. The mutual information $I(S;Y)$ was calculated using the mutual information neural estimator (MINE) with a moving average bias corrector with a an exponential rate of 0.1 \cite{belghazi2018mutual}; the resulting information was averaged over the last 100 iterations. The neural networks employed were a 2-hidden layer multi-layer perceptron with 100 ReLU6 activation functions for all the datasets and tasks, except from the example on the Colored MNIST dataset, where the \texttt{CNN-mine} from Table \ref{table:cnn_architectures} was used. In all tasks the input was a concatenation of $S$ and $Y$, except from the example on the Colored MNIST dataset, where in all convolutional layers the private data $S$ was added as a bias and in all linear layers $S$ was concatenated to the input. The hyperparameters used to train the networks are displayed in Table \ref{table:hyperparameters_mine}.

\paragraph{Group fairness and utility indicators} The accuracy on $T$ and $S$, the discrimination, and the error and equalized odds gaps were calculated using both the input data $X$ and the generated representations $Y$. They were calculated with a Logistic regression (LR) classifier and a random forest (RF) classifier with the default settings from \texttt{scikit-learn} \cite{scikit-learn}. The prior displayed on the accuracy on $T$ and $S$ figures is the accuracy of a classifier that only infers the majority class of $T$ and $S$, respectively, from the training dataset.

\paragraph{Comparison with PPVAE \cite{nan2020variational}} We implemented the algorithm as described in the original paper. We used the same neural network architecture (i.e., multi-layer perceptron with a single hidden layer with 100 units and ReLU activations) and training hyperparameters than in our method in order to provide a fair comparison. The original method was not prepared to handle other data that was not continuous, so we expanded the method assuming categorical, and Bernoulli output distributions for the non-continuous data. We studied 30 linearly equiespaced values of the hyperparameter $\eta^{-1}$ from~\eqref{eq:eq_ppvae} ranging from 1 to 50. 

\paragraph{Comparison with VFAE \cite{louizos2015variational}} We implemented the algorithm as described in the original paper. We used the same neural network architecture (i.e., multi-layer perceptron with a single hidden layer with 100 units and ReLU activations) and training hyperparameters than in our method in order to provide a fair comparison.  We studied 30 linearly equiespaced values of the hyperparameter $\delta$ multiplying the MMD term ranging from $N_{\textnormal{batch}}$ to $N_{\textnormal{batch}} 1000$. We calculated the MMD using random kitchen sinks \cite{rahimi2009weighted} with $D = 500$ and $\gamma = 1$ as in the original paper.

\paragraph{Comparison with LFR \cite{zemel2013learning}} We implemented the algorithm as described in the original paper. We used 10 prototypes ($K=10$) and adjusted the hyperparemeters $A_x$, $A_y$, and $A_z$ so that the L-BFGS found a feasible solution in $150,000$ iterations with a tolerance of $10^{-5}$. We observed that different values of the hyperparameters (as noted by \cite{zemel2013learning}) produced almost equivalent results, so we only report here the results for $A_x = 10
^{-4}$, $A_y = 0.1$, and $A_z = 500$.

\paragraph{Comparison with FFVAE \cite{creager2019flexibly}} We implemented the algorithm as described in the original paper. We used the same neural network architecture (i.e., multi-layer perceptron with a single hidden layer with 100 units and ReLU activations) and training hyperparameters than in our method in order to provide with a fair comparison. Since the sensitive variable has dimension 1, there was no need for an adversary discriminator and therefore the $\gamma$ parameter was not explored. We performed experiments with different values of $\alpha$ ranging from $1$ to $400$ as in the original paper and we observed that the value of the hyperparameter did not modify much the results, as shown in Table \ref{table:ffvae_alpha_performance}. We selected $\alpha = 200$ for the comparisons since the results for that $\alpha$ were the ones with the better trade-off between accuracy and fairness.

\paragraph{Comparison with CFAIR \cite{zhao2019conditional}} We implemented the algorithm as described in the original paper. For the experiments in the Adult dataset, we used the same neural network architecture (i.e., multi-layer perceptron with a single hidden layer with 100 units and ReLU activations) than in our method  for the encoder, decoder, and the adversarial decoders, which are equipped with a gradient reversal layer \cite{ganin2016domain}. Also, we employed the same training hyperparameters to ensure a fair comparison between the two methods. We performed experiments with different values of $\lambda$ ranging from $0.1$ to $1000$ as in the original paper. However, for the experiments in the COMPAS dataset, we could not obtain good results with our architecture and hyperparameters nor could we replicate their results with their architecture and hyperparameters. Hence, for this dataset, we report here the results displayed in the original paper.

\section{Group fairness and utility indicators}
\label{app:fairness_indicators}

In this section of the appendix, we define and put into perspective a series of metrics, employed in this article, that indicate the predicting and group fairness quality of a classifier.

\paragraph{Notation} Let $X \in \mathcal{X}$, $S \in \mathcal{S}$, and $T \in \mathcal{T}$ be random variables denoting the input data, the sensitive data, and the target task data, respectively. Let also $\mathcal{X} \subseteq \mathbb{R}^{d_{X}}$, $\mathcal{S} = \lbrace 0, 1 \rbrace$, and $\mathcal{T} = \lbrace 0, 1 \rbrace$. Let $w: \mathcal{X} \rightarrow \mathcal{T}$ be a classifier; that is, $w \in \mathcal{W}$ is a function that receives as an input an instance of the input data $x \in \mathcal{X}$ and outputs an inference about the target task value $t \in \mathcal{T}$ for that input data. Let us also consider the setting where we have a dataset that contains $N$ samples of the random variables, i.e., $D = \lbrace (x^{(i)}, s^{(i)}, t^{(i)}) \rbrace_{i=1}^N$. Finally, let $\hat{P}$ denote the empirical probability distribution on the dataset $D$, $\hat{P}_{S=\sigma}$ the empirical probability distribution on the subset of the dataset $D$ where $s^{(i)} = \sigma$, i.e., $\lbrace (x, s, t) \in D : s = \sigma \rbrace$, and $\hat{P}_{(S=\sigma, T=\tau)}$ the empirical probability distribution on the subset of the dataset $D$ where $s^{(i)} = \sigma$ and $t^{(i)} = \tau$, i.e., $\lbrace (x, s, t) \in D : s = \sigma \textnormal{ and } t = \tau \rbrace$.

A common metric to evaluate the performance (utility) of a classifier $w$ on a dataset is its \emph{accuracy}, which measures the fraction of correct classifications of $w$ on such a dataset.
\begin{definition} The \emph{accuracy} of a classifier $w$ on a dataset $D$ is
\begin{equation}
    \textnormal{Accuracy}(w,D) = \hat{P}(w(X) = T).
\end{equation}
\end{definition}
An ideally fair classifier $w$ would maintain \emph{demographic parity} (or \emph{statistical parity}) and \emph{accuracy parity}, which, respectively, mean that $w(X) \perp S$ (or, equivalently if $w$ is deterministic, that $\hat{P}_{S=0}(w(X) = 1) = \hat{P}_{S=1}(w(X) = 1)$) and that $\hat{P}_{S=0}(w(X) \not = T) = \hat{P}_{S=1}(w(X) \not = T)$ \cite{zhao2019conditional}. In other words, if a classifier has demographic parity, it means that it gives a positive outcome with equal rate to the members of $S=0$ and $S=1$. However, demographic parity might damage the desired utility of the classifier \cite{hardt2016equality}, \cite[Corollary 3.3]{zhao2019inherent}. Accuracy parity, on the contrary, allows the existance of perfect classifiers \cite{zhao2019conditional}. The metrics that assess the deviation of a classifier from demographic and accuracy parities are the \emph{discrimination} or \emph{demographic parity gap} \cite{zemel2013learning, zhao2019conditional} and the \emph{error gap} \cite{zhao2019conditional}.
\begin{definition} The \emph{discrimination} or \emph{demographic parity gap} of a classifier $w$ to the sensitive variable $S$ on a dataset $D$ is
\begin{align}
    \textnormal{Discrimin}\textnormal{ation}(w,D) = \nonumber \left| \hat{P}_{S=0}(w(X) = 1) - \hat{P}_{S=1}(w(X) = 1)\right|.
\end{align}
\end{definition}
\begin{definition} The \emph{error gap} of a classifier $w$ with respect to the sensitive variable $S$ on a dataset $D$ is
\begin{align}
    \textnormal{Error gap}(w,D) = \nonumber \left| \hat{P}_{S=0}(w(X) \not = T) - \hat{P}_{S=1}(w(X) \not = T)\right|.
\end{align}
\end{definition}
Another advanced notion of fairness is that of \emph{equalized odds} or \emph{positive rate parity}, which means that $\hat{P}_{(S=0,T=\tau)}(w(X)=1) = \hat{P}_{(S=1,T=\tau)}(w(X)=1)$, for all $\tau \in \lbrace 0, 1 \rbrace$ or, equivalently, that $w(X)~\perp~S~|~T$ \cite{hardt2016equality}. This notion of fairness requires that the true positive and false positive rates of the groups $S=0$ and $S=1$ are equal. The metric that assesses the deviation of a classifier from equalized odds is the \emph{equalized odds gap} \cite{zhao2019conditional}.
\begin{definition} The \emph{equalized odds gap} of a classifier $w$ with respect to the sensitive variable $S$ on a dataset $D$ is
\begin{align*}
    \textnormal{Equalized odds gap}(w,D) = 
    \max_{\tau \in \lbrace 0, 1 \rbrace} \left| \hat{P}_{(S=0,T=\tau)}(w(X) = 1) - \hat{P}_{(S=1,T=\tau)}(w(X) = 1) \right|.
\end{align*}
\end{definition}
\begin{remark}
\label{remark:fairness_indicators_representations}
In the particular case of learning fair representations, the classifier $w: \mathcal{X} \rightarrow \mathcal{T}$ consists of two stages: an encoder $w_{\textnormal{enc}}: \mathcal{X} \rightarrow \mathcal{Y}$ and a decoder $w_{\textnormal{dec}}: \mathcal{Y} \rightarrow \mathcal{T}$, where the intermediate variable $Y = w_{\textnormal{enc}}(X)$ is the fair representation of the data. Therefore:
\begin{enumerate}
    \item Minimizing $I(S;Y)$ encourages demographic parity, since 
    \begin{equation*}
        I(S;Y) = 0 \iff Y~\perp~S \implies w(X)~\perp~S.
    \end{equation*}
    \item Minimizing $I(S;Y|T)$ encourages equalized odds, since 
    \begin{equation*}
        I(S;Y|T) = 0 \iff Y~\perp~S~|~T \implies w(X)~\perp~S~|~T.
    \end{equation*}
\end{enumerate}
\end{remark}

\begin{remark}
\label{remark:cfb_and_cpf_encourage_demographic_parity}
Based on Remark~\ref{remark:fairness_indicators_representations}, we note that the variational approach to the CFB and the CPF for generating private and/or fair representations encourages demographic parity, since the minimization of the Lagrangians of such problems, $\mathcal{L}_{\textnormal{CFB}}$ and $\mathcal{L}_{\textnormal{CPF}}$, indeed minimizes $I(S;Y)$.
\end{remark}
\begin{remark}
\label{remark:cfb_and_cpf_equalized_odds}
Contrary to Remark~\ref{remark:cfb_and_cpf_encourage_demographic_parity} concerning the minimization of the demographic parity gap, we cannot say that the variational approach to the CFB and the CPF minimizes the equalized odds gap. 

Even though $I(S;Y) = I(S;Y|T) + I(S;Y;T)$, since $I(S;Y;T)$ can be negative \cite{yeung1991new}, then $I(S;Y)$ is not necessarily greater than $I(S;Y|T)$ and thus there is no guarantee that minimizing $I(S;Y)$ will minimize $I(S;Y|T)$ as well.

Therefore, the minimization of $\mathcal{L}_{\textnormal{CFB}}$ and $\mathcal{L}_{\textnormal{CPF}}$, which minimizes $I(S;Y)$, also minimizes the equalized odds gap when $I(S;Y;T) \geq 0$ and therefore $I(S;Y) \geq I(S;Y|T)$. That is, the CFB and the CPF minimize the equalized odds gap when there is no synergy between $X$ and $Y$ to learn about $S$; i.e., $I(S;Y) + I(S;X) \geq I(S;(X,Y))$.
\end{remark}
Continuing with the reflection on Remarks~\ref{remark:fairness_indicators_representations} and~\ref{remark:cfb_and_cpf_equalized_odds}, a constrained optimization formulation \textit{a la} CPF that leads to a minimization of the equalized odds gap is possible. Namely, 
\begin{equation*}
    \arg \inf_{P_{Y|X}} \left \lbrace I(S;Y|T) + I(X;Y|S,T) \right \rbrace \textnormal{ s.t. } I(T;Y|S) \geq r,
\end{equation*}
which minimizes the dark and darker gray areas from Figure~\ref{fig:i_diagram_cpf}, which correspond to the sensitive and irrelevant data, respectively, except the intersection between the sensitive data $S$, the representations $Y$, and the target $T$, which corresponds to $I(S;Y;T)$. In this formulation, it is ensured that the representations $Y$ also maintain a certain level of the information $r$ of the target that is not shared in the sensitive data $I(T;Y|S)$.

\section{Non-convexity of the CPF and the CFB}
\label{app:convexity_problems}
In this section of the appendix, we show how both the CPF and the CFB as defined in~\eqref{eq:conditional_privacy_funnel_opt} and~\eqref{eq:conditional_fairness_bottleneck_opt} are non-convex optimization problems. 
\begin{lemma} Let $X \in \mathcal{X}$, $S \in \mathcal{S}$, $Y \in \mathcal{Y}$, and $T \in \mathcal{T}$ be random variables. Then, 
\begin{enumerate}
    \item If the Markov chain $S \leftrightarrow X \rightarrow Y$ holds and the distributions of $X$ and $S$ are fixed, then $I(X;Y)$, $I(S;Y)$, and $I(X;Y|S)$ are convex functions with respect to the density $p_{Y|X}$.
    \item If, additionally, the Markov chain $T \leftrightarrow X \rightarrow Y$ holds and the distributions of $X$, $S$, and $T$ are fixed, then $I(X;Y|S,T)$ and $I(T;Y|S)$ are also convex functions with respect to the density $p_{Y|X}$.
\end{enumerate}
\label{lemma:mi_conv_funcs_pyx}
\end{lemma}
\begin{proof}
We start the proof leveraging \cite[Theorem 2.7.4]{cover2012elements}, which, in our setting, tells us that:
\begin{itemize}
\item $I(X;Y)$ is a convex function of $p_{Y|X}$ if $p_{X}$ is fixed. 
\item $I(S;Y)$ is a convex function of $p_{Y|S}$ if $p_{S}$ is fixed.
\item $I(X;Y|S)$ is a convex function of $p_{Y|X,S}$ if $p_{X|S}$ is fixed.
\item $I(X;Y|S,T)$ is a convex function of $p_{Y|X,S,T}$ if $p_{X|S,T}$ is fixed.
\item $I(T;Y|S)$ is a convex function of $p_{Y|S,T}$ if $p_{T|S}$ is fixed.
\end{itemize}
Then, since $p_{Y|S} = \mathbb{E}_{p_{X|S}}[p_{Y|X}]$, $p_{Y|X,S} = \left(\frac{p_{X|S}}{p_X}\right) p_{Y|X}$, $p_{Y|X,S,T} = \left(\frac{p_{X|S,T}}{p_X}\right) p_{Y|X}$, and $p_{Y|S,T} = \mathbb{E}_{p_{X|S,T}}[p_{Y|X}]$ are non-negative weighted sums as defined in \cite[2.2.1]{boyd2004convex}, they preserve convexity. Hence, $I(S;Y)$, $I(X;Y|S)$, $I(X;Y|S,T)$, and $I(T;Y|S)$ are convex functions of $p_{Y|X}$, if $p_S$, $p_{X|S}$, $p_{X|S,T}$, and $p_{T|S}$ are fixed, respectively.
\end{proof}
\begin{proposition}
\label{prop:non_convexity_cpf}
Let us consider that the distributions of $S$ and $X$ are fixed and that the conditional distribution $P_{Y|X}$ has a density $p_{Y|X}$. Then, the CPF optimization problem is not convex.
\end{proposition}
\begin{proof}
From Lemma \ref{lemma:mi_conv_funcs_pyx} we know that $I(S;Y)$ and $I(X;Y|S)$ are convex functions with respect to $p_{Y|X}$ for fixed $p_{S}$ and $p_{X|S}$. Hence, the constraint $I(X;Y|S) \geq r$ is concave.
\end{proof}
\begin{proposition}
\label{prop:non_convexity_cfb}
Let us consider that the distributions of $S$, $T$, and $X$ are fixed and that the conditional distribution $P_{Y|X}$ has a density $p_{Y|X}$. Then, the CFB optimization problem is not convex.
\end{proposition}
\begin{proof}
From Lemma \ref{lemma:mi_conv_funcs_pyx} we know that $I(S;Y)$, $I(X;Y|S,T)$, and $I(T;Y|S)$ are convex functions with respect to $p_{Y|X}$ for fixed $p_{S}$, $p_{X|S,T}$, and $p_{T|S}$. Hence, the constraint $I(T;Y|S) \geq r$ is concave.
\end{proof}

\section{Limitations and future directions}
\label{app:limitations}

\paragraph{Problem formulation} 
The CPF~\eqref{eq:conditional_privacy_funnel_opt} and the CFB~\eqref{eq:conditional_fairness_bottleneck_opt} are non-convex optimization problems with respect to $P_{Y|X}$ (see Appendix~\ref{app:convexity_problems}). Therefore, (i) the optimal conditional distribution $P^{\star}_{Y|X}$ that minimizes the Lagrangian might not be achieved through gradient descent, and (ii) even if $P^{\star}_{Y|X}$ is achieved, it could be a sub-optimal value for~\eqref{eq:conditional_privacy_funnel_opt} or~\eqref{eq:conditional_fairness_bottleneck_opt}, since the problems are not \emph{strongly dual} \cite[Section 5.2.3]{boyd2004convex}. A possible solution could be the application of a monotonically increasing concave function $u$ to $I(X;Y|S)$ or $I(T;Y|S)$ in the CPF or CFB Lagrangians, respectively, so that $u(I(X;Y|S))$ or $u(I(T;Y|S))$ is concave (and hence the Lagrangian is convex) in the domain of interest. For some $u$, this approach might allow to attain the desired $r$ in~\eqref{eq:conditional_privacy_funnel_opt} or~\eqref{eq:conditional_fairness_bottleneck_opt} with a specific value of the Lagrange multiplier; see \cite{rodriguez2020convex} for an example of this approach for the IB.
\paragraph{Proposed approach} The proposed approach entails two limitations that are common in variational attempts at solving an optimization problem. Namely: (i) it approximates the decoding and the marginal distributions and (ii) it  considers parametrized densities. The first issue restricts the search space of the possible encoding distributions $P_{Y|X}$ to those distributions with a decoding and marginal distributions that follow the restrictions of the variational approximation. The second issue further limits the search space to the obtainable encoding distributions with densities $p_{Y|X,\theta}$ with a parametrization $\theta$. For this reason, richer encoding distributions and marginals, e.g., by means of normalizing flows \cite{rezende2015variational, kingma2016improved}, are a possible direction to mitigate these issues.


\end{document}